\documentclass[10pt,twocolumn]{IEEEtran}

\usepackage{cite}
\usepackage{llncsdoc}
\usepackage{times}
\usepackage{helvet}
\usepackage{courier}
\usepackage{amsthm}
\usepackage{amssymb,amsmath}
\usepackage{dsfont}
\usepackage{graphicx}
\usepackage{url}
\usepackage{bm}
\usepackage{algorithm,algorithmic}
\theoremstyle{definition}
\newtheorem{thm}{Theorem}
\newtheorem{cor}{Corollary}
\newtheorem{defn}{Definition}

\newtheorem{pro}{Proposition}

\newtheorem{lemma}{Lemma}
\DeclareMathOperator{\argmin}{argmin}
\DeclareMathOperator{\argmax}{argmax}
\DeclareMathOperator{\diam}{diam}
\DeclareMathOperator{\ddim}{ddim}

\DeclareMathOperator{\fat}{fat}
\DeclareMathOperator{\sign}{sign}

\begin{document}
\title{Metric Learning via Maximizing the Lipschitz Margin Ratio}
\author{Mingzhi~Dong, Xiaochen~Yang, Yang~Wu, Jing-Hao~Xue
\thanks{M.~Dong, X.~Yang and J.-H.~Xue are with the Department of Statistical Science, University College London, UK (e-mail: mingzhi.dong.13@ucl.ac.uk; xiaochen.yang.16@ucl.ac.uk; jinghao.xue@ucl.ac.uk).}
\thanks{Y.~Wu is with the Institute for Research Initiatives, Nara Institute of Science and Technology, Japan (e-mail: yangwu@rsc.naist.jp).}}

%%%%%%%%%\markboth{A manuscript to IEEE TNNLS}{Dong\MakeLowercase{~\textit{et al.}}: Lipschitz Margin Ratio}

\maketitle
\begin{abstract}
In this paper, we propose the Lipschitz margin ratio and a new metric learning framework for classification through maximizing the ratio. This framework enables the integration of both the inter-class margin and the intra-class dispersion, as well as the enhancement of the generalization ability of a classifier. To introduce the Lipschitz margin ratio and its associated learning bound, we elaborate the relationship between metric learning and Lipschitz functions, as well as the representability and learnability of the Lipschitz functions. After proposing the new metric learning framework based on the introduced Lipschitz margin ratio, we also prove that some well known metric learning algorithms can be shown as special cases of the proposed framework. In addition, we illustrate the framework by implementing it for learning the squared Mahalanobis metric, and by demonstrating its encouraging results on eight popular datasets of machine learning.
\end{abstract}


\section{Introduction}\label{s:intro}

Classification is a fundamental area in machine learning. 
For classification, it is crucial to appropriately measure the distance between instances. 
One of the established classifier, the nearest neighbor (NN) classifier, classifies a new instance into the class of the training instance with the shortest distance.  

In practice it is often difficult to handcraft a well-suited and adaptive distance metric. 
To mitigate this issue, metric learning has been proposed to enable learning a metric automatically from the data available. 
Metric learning with a convex objective function was first proposed in the pioneering work of~\cite{xing2002distance}. 
The large margin intuition was introduced into the research of metric learning by the seminal ``large margin metric learning" (LMML)~\cite{schultz2004learning} and ``large margin nearest neighbor" (LMNN)~\cite{weinberger2009distance}. 
Besides the large margin approach, other inspiring metric learning strategies have been developed, such as nonlinear metrics~\cite{kedem2012non,hu2014discriminative}, localized strategies~\cite{dong2017lam3l,wang2014globality,noh2018generative} and scalable/efficient algorithms~\cite{shen2014efficient,qian2015efficient}. 
Metric learning has also been adopted by many other learning tasks, 
such as semi-supervised learning~\cite{ying2017manifold}, unsupervised-learning~\cite{jia2016new}, multi-task/cross-domain learning~\cite{luo2017heterogeneous,wang2018cross}, AUC optimization~\cite{huo2018cross} and distributed approaches~\cite{li2015distributed}.
%such as multi-task learning~\cite{yang2013geometry} and semi-supervised learning~\cite{nicolae2015joint}.
 
On top of the methodological and applied advancement of metric learning, some theoretical progress has also been made recently, in particular on deriving different types of generalization bounds for metric learning~\cite{jin2009regularized,  guo2014guaranteed, verma2015sample, cao2016generalization}.
These developments have theoretically justified the performance of metric learning algorithms. 
However, they generally lack a geometrical link with the classification margin, not as interpretable as one may expect (e.g.~like the clear relationship between margin and $1/|w|$ in support vector machines (SVM)).

\begin{figure}
\centering
\includegraphics[width=0.5\textwidth]{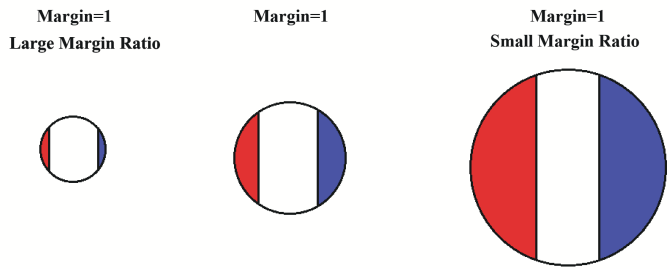}
\caption{An illustration of the margin ratio. Each ball indicates a metric space. The red area indicates the area of positive class instances; the blue area indicates the area of negative class instances. Although the margins between the two classes in different metric spaces are the same, it is intuitive that the difficulties of classification are distinct in different metric spaces.} 
\label{margin_ratio_ball}
\end{figure}

Besides the inter-class margin, the intra-class dispersion is also crucial to classification~\cite{flamary2016wasserstein,do2013convex,jebara2009relative}. 
The intra-class dispersion is especially important for metric learning, because different metrics may lead to similar inter-class margins and quite different intra-class dispersion. 
As illustrated in Figure~\ref{margin_ratio_ball}, although the margins in those different metric spaces are exactly the same, the classification becomes more difficult as the margin ratio decreases. Therefore, the seminal work of~\cite{xing2002distance} and many later work made efforts to consider the inter-class margin and the intra-class dispersion at the same time. 

In this paper, we aim to propose a new concept, the Lipschitz margin ratio, to integrate both inter-class and intra-class properties, and through maximizing the Lipschitz margin ratio we aim to propose a new metric learning framework to enable the enhancement of the generalization ability of a classifier. These two novelties are our main contributions to be made in this work.

To achieve these two aims and present our contributions in a well-structured way, we organize the rest of this paper as follows.
Firstly, in Section~\ref{s:Lip-and-distance} we discuss the relationship between the distance-based classification / metric learning and Lipschitz functions. We show that a Lipschitz extension, which is a distance-based function, can be regarded as a generalized nearest neighbor model, which enjoys great representation ability. 
Then, in Section~\ref{s:Lip-ratio} we introduce the Lipschitz margin ratio, and we point out that its associated learning bound indicates the desirability of maximizing the Lipschitz margin ratio, for enhancing the generalization ability of Lipschitz extensions. 
Consequently in Section~\ref{s:framework}, we propose a new metric learning framework through maximizing the Lipschitz margin ratio. Moreover, we prove that many well known metric learning algorithms can be shown as special cases of the proposed framework. 
Then for illustrative purposes, we implement the framework for learning the squared Mahalanobis metric. The method is presented in Section~\ref{ss:special}, and its experimental results in Section~\ref{s:experiments}, which demonstrate the superiority of the proposed method. 
Finally, we draw conclusions and discuss future work in Section~\ref{s:conclusion}. 
For the convenience of readers, some theoretical proofs are deferred to the Appendix.

\section{Lipschitz Functions and Distance-based Classifiers}\label{s:Lip-and-distance}

\subsection{Definition of Lipschitz Functions}

To start with, we will review the definitions of Lipschitz functions, the Lipschitz constant and the Lipschitz set.

\begin{defn}
\cite{weaver1999lipschitz}  Let $(\mathcal X, \rho_{\mathcal X})$ be a metric space. A function $f: \mathcal X \rightarrow \mathds R$ is called \emph{Lipschitz continuous} if $\exists C < \infty, \forall \bm x_1, \bm x_2 \in \mathcal X$,
\begin{equation*}
%\begin{split}
     | f(\bm x_1)- f(\bm x_2)| \leq C \rho_{\mathcal X}(\bm x_1, \bm x_2).
%\end{split}
\end{equation*}
The \emph{Lipschitz constant} $L(f)$ of a Lipschitz function $f$ is
\begin{equation*}
\begin{split}
  & L(f) \\
  = & \inf\{ C\in \mathds{R} | \forall \bm x_1, \bm x_2 \in \mathcal X, |f(\bm x_1)- f(\bm x_2)| \leq C \rho_{\mathcal X}(\bm x_1, \bm x_2) \}  \\
    = & \sup_{\bm x_1, \bm x_2 \in \mathcal X: \bm x_1 \neq \bm x_2} \frac{|f(\bm x_1)- f(\bm x_2)|}{\rho_{\mathcal X}(\bm x_1, \bm x_2)},
\end{split}
\end{equation*}
and function $f$ is also called a $L$-\emph{Lipschitz function} if its Lipschitz constant is $L$. Meanwhile, all $L$-Lipschitz functions construct the $L$-\emph{Lipschitz set}
\begin{equation*}
    L\mbox{-}Lip(\mathcal X)= \{f: \mathcal X \rightarrow \mathds R; L(f) \leq L\}.
\end{equation*}
\end{defn}

From the definitions, we can observe that the Lipschitz constant is fundamentally connected with the metric $\rho_{\mathcal X}$; and that the Lipschitz functions have specified a family of ``smooth'' functions, whose change of output values can be bounded by the distances in the input space.

\subsection{Lipschitz Extensions and Distance-based Classifiers}

Distance-based classifiers are the classifiers that are based on certain kinds of distance metrics. 
Most of distance-based classifiers stem from the nearest neighbors (NN) classifier. 
To decide the class label of a new instance, the NN classifier compares the distances between the new instance and the training instances.

In binary classification tasks, a Lipschitz function is commonly used as the classification function $f$ and the instance ${\bm x}$ is then classified according to the sign of $f({\bm x})$. 
Using Theorem~\ref{thm:MWext}, we shall present a family of Lipschitz functions, called Lipschitz extensions. 
We shall also show that Lipschitz extensions present a distance-based classifier, and that a special case of Lipschitz extensions returns exactly the same classification result as the NN classifier.

\begin{thm}
\cite{mcshane1934extension,whitney1934analytic,weaver1999lipschitz,luxburg2004distance} (McShane-Whitney Extension Theorem)
Given a function $u$ defined on a finite subset $A=\{\bm x_1,\dots, \bm x_n\}$, there exist a family of functions which coincide with $u$ on $\bm x_1, \dots, \bm x_n$, are defined on the whole space $\mathcal X$, and have the same Lipschitz constant as $u$. Additionally, it is possible to explicitly construct $u$ in the following form and they are called \emph{$L$-Lipschitz extensions} of $u$:
\begin{equation*}
U_{\alpha}(\bm x) = \alpha  U_1(\bm x) + (1-\alpha) U_2(\bm x), 
\end{equation*}
where $\alpha \in [0,1]$, 
\begin{align*}
    U_1(\bm x)&= \overline u(\bm x) = \inf_{\bm a \in A}\{ u(\bm a)  + L \rho(\bm x, \bm a)\},\\
    U_2(\bm x)&= \underline u(\bm x) = \sup_{\bm a \in A}\{ u(\bm a ) - L \rho(\bm x, \bm a)\}.
\end{align*}
\label{thm:MWext}
\end{thm}

Theorem~\ref{thm:MWext} can be readily validated by calculating the values of $U_1(\bm x)$ and $U_2(\bm x)$ on the finite points $\bm x_1, \dots, \bm x_n$.  
The bound of the Lipschitz constant of $\overline u(\bm x)$ and $\underline u(\bm x)$ can be proved on the basis of the Lemmas in Appendix.

Theorem~\ref{thm:MWext} clearly shows that Lipschitz extensions are distance-based function.
Moreover, we can illustrate the relationship between Lipschitz extension functions and empirical risk as follows. 

Assume $A$ is the set of training instances of a classification task $A=\{\bm x_1,\dots, \bm x_N\}$. If there are no $\bm x_i, \bm x_j$ such that $\rho(\bm x_i,\bm x_j)=0$ while their labels $t_i \neq t_j$ (i.e. no overlap between training instances from different classes), setting $u(\bm x_i) = t_i$ would result in zero empirical risk, and $u(\bm x_i)$ would be a Lipschitz function with Lipschitz constant $L_0$,
\begin{equation*}
    L_0 = \sup_{i,j} \frac{|t_i-t_j|}{\rho(\bm x_i,\bm x_j)},
\end{equation*}
where the existence of such a function $u$, i.e.~the Lipschitz extensions, is guaranteed by Theorem~\ref{thm:MWext}. 

That is, when doing classification, if we set $L$ of Lipschitz extension to be larger than $L_0$, zero empirical risk could be obtained. In other words, as distance-based functions, Lipschitz extensions enjoy excellent representation ability for classification tasks. 

Moreover, if we set $\alpha$ as $1/2$, Lipschitz extensions will have exactly the same classification results as the NN classifier:
\begin{pro} \cite{luxburg2004distance}
The function $U_{1/2}(\bm x)$ defined above has the same sign, i.e.~has the same classification results, as that of the NN classifier.
\end{pro}

\section{Lipschitz Margin Ratio}\label{s:Lip-ratio}

In the previous section, we show that Lipschitz extensions can be viewed as a distance-based classifier, and its representation ability is so strong that zero empirical error can be obtained under mild conditions. In this section, we shall propose the Lipschitz margin ratio to control the model complexity of the Lipschitz functions and hence improve its generalization ability. To start with, we propose an intuitive way to understand the Lipschitz margin and the Lipschitz margin ratio. Then, learning bounds of the Lipschitz margin ratio will be presented.

\subsection{Lipschitz Margin}

We define the training set of class $k$ as $\bm S_{k}=\{\bm x_i | t_i = k, \bm x_i \in \bm S\}$, where $k\in\{1,-1\}$; %$\bm S_{\xi=0, f}=\{\bm x_i | t_i f(\bm x_i) \geq 1\}$;
%$\bm S_{c, \xi=0, f}=\{\bm x_i | t_i f(\bm x_i) \geq 1,  t_i = c\}$;
the decision boundary of classification function $f$ as $\bm H_f=\{\bm h| \bm h \in \mathcal X, f(\bm h)=0\}$.
The margin used in \cite{luxburg2004distance} is equivalent to the Lipschitz margin defined below.

\begin{defn}
The \emph{Lipschitz margin} is the distance between the training sets $\bm S_{1}$ and $\bm S_{-1}$:
%between the hyperplane and the training instances $d(\bm S, \bm H)$
%The Lipschitz hard margin in \cite{luxburg2004distance} is defined as the shortest distance
%between the hyperplane and the training instances $d(\bm S, \bm H)$
\begin{equation}
\mbox{L-Margin}= D(\bm S_{1}, \bm S_{-1})= \min\limits_{\bm x_i \in \bm S_{-1}, \bm x_j\in \bm S_{1} } \rho(\bm x_i, \bm x_j).
\end{equation}
\end{defn}

The relationship between the Lipschitz margin and the Lipschitz constant is established as follows.

\begin{pro}
\label{margin_lipcons}
For any $L$-Lipschitz function $f$ satisfying $\forall \bm x_i \in \bm S_{1}, f(\bm x_i) \geq 1$ and $\forall \bm x_j \in \bm S_{-1}, f(\bm x_j) \leq -1$,
\begin{equation}
\mbox{L-Margin} \geq \frac{2}{L(f)}. 
\end{equation}
\end{pro}
\begin{proof}
Let $\bm x_n$ and $\bm x_m$ denote the nearest instances from different classes, i.e.
\begin{equation*}
\rho(\bm x_n, \bm x_m) = D(\bm S_{1}, \bm S_{-1})
 = \min\limits_{\bm x_i \in \bm S_{-1}, \bm x_j\in \bm S_{1} } \rho(\bm x_i, \bm x_j).
\end{equation*}
It is straightforward to see
\begin{equation*}
\begin{split}
    \frac{2}{L(f)} &\leq \frac{2}{|f(\bm x_n)- f(\bm x_m)|/ \rho (\bm x_n, \bm x_m)} \\
    &\leq \rho (\bm x_n, \bm x_m) \\
    &=D(\bm S_{1}, \bm S_{-1}),
\end{split}
\end{equation*}
where the first inequality follows from the definition of the Lipschitz constant; and the second inequality is for the reason that $\forall \bm x_i \in \bm S_{1}, f(\bm x_i) \geq 1$ and $\forall \bm x_j \in \bm S_{-1}, f(\bm x_j) \leq -1$, then $|f(\bm x_n)- f(\bm x_m)| \geq 2$.
\end{proof}

The proposition shows that the Lipschitz margin can be lower bounded by the multiplicative inverse Lipschitz constant.

\begin{figure}
\centering
\includegraphics[width=0.3\textwidth]{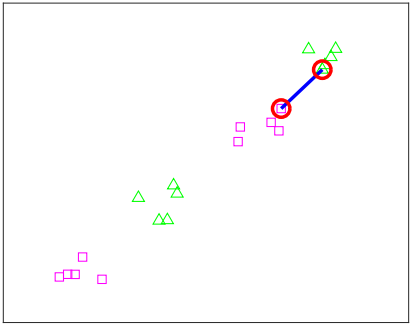}
\caption{An illustration of the Lipschitz margin. Green triangles are instances from the positive class, and purple squares are from the negative class. Data points with red circles around them are the nearest instances from different classes. The length of the blue line indicates the value of the Lipschitz margin.}
\label{margin2_hard}
\end{figure}

The Lipschitz margin is closely related to the margin adopted in SVM (the distance between the hyperplane $\bm H$ and the training instances $\bm S$), 
\begin{equation*}
D(\bm S, {\bm H}_f)= \min\limits_{\bm x_i \in \bm S, \bm h\in {\bm H}_f } \rho(\bm x_i, \bm h),
\end{equation*}
As illustrated in Figure \ref{margin2_hard}, the Lipschitz margin is also suitable for the classification of non-linearly separable classes. The relationship between these two types of margins are described via the following proposition.

\begin{pro}
\label{pro_euclidean}
In the Euclidean space, let $f$ be any continuous function which correctly classifies all the training instances, i.e. $\forall \bm x_i \in \bm S, t_i f(\bm x_i) \geq 1$, then
\begin{equation*}
    D(\bm S_{1}, \bm S_{-1}) \geq 2 D(\bm S, \bm H).
\end{equation*}
\end{pro}

\begin{proof}
In the Euclidean space, 
\begin{equation*}
D(\bm S_{1}, \bm S_{-1})= \min\limits_{\bm x_i \in \bm S_{-1}, \bm x_j\in \bm S_{+1} } \rho_E(\bm x_i, \bm x_j),
\end{equation*}
\begin{equation*}
D(\bm S, {\bm H}_f)= \min\limits_{\bm x_i \in \bm S, \bm h\in {\bm H}_f } \rho_E(\bm x_i, \bm h),
\end{equation*}
and $\rho_E(\bm x_i, \bm x_j)=\sqrt{(\bm x_i- \bm x_j)^T (\bm x_i- \bm x_j) }$ is the Euclidean distance.

Let $\bm x_n$ and $\bm x_m$ denote the nearest instances from different classes, i.e.
\begin{equation*}
\rho_E(\bm x_n, \bm x_m)= D(\bm S_{1}, \bm S_{-1})= \min\limits_{\bm x_i \in \bm S_{-1}, \bm x_j\in \bm S_{+1} } \rho_E(\bm x_i, \bm x_j), 
\end{equation*}
where $\bm x_n \in \bm S_{-1}, \bm x_m \in \bm S_{+1}$.

We define a connected set $\bm Z= \{a\bm x_n+ (1-a)\bm x_m | 0 \leq a \leq 1\}$, which indicates the line segment between $\bm x_n$ and $\bm x_m$. Because $f(\bm x_n)\leq -1$, $f(\bm x_m)\geq 1$ and for any continuous function $f$, it maps connected sets into connected sets,  there exists $\bm z\in \bm Z$, such that $f(\bm z) = 0$. According to the definition of ${\bm H}_f$, we can see $\bm z \in {\bm H}_f$.
Therefore,
\begin{equation*}
    \begin{split}
    D(\bm S, {\bm H}_f)& = \min\limits_{\bm x_i \in \bm S, \bm h\in {\bm H}_f } \rho_E(\bm x_i, \bm h)\\
        & \leq \min\limits_{\bm x_i \in \bm S} \rho_E(\bm x_i, \bm z)\\
        & \leq \frac{\rho_E(\bm x_n, \bm z) + \rho_E(\bm x_m, \bm z)}{ 2}\\
        & = \frac{\rho_E(\bm x_n, \bm x_m)}{ 2}\\
        & = \frac{D(\bm S_{1}, \bm S_{-1})}{2},
    \end{split}
\end{equation*}
where the second equality follows from the connectedness property of $\bm Z$.
\end{proof}

\subsection{Lipschitz Margin Ratio}

The Lipschitz margin discussed above effectively depicts the inter-class relationship. However, as we mentioned before, when we learn the metrics, different metrics will result in different intra-class dispersion and it is also important to consider intra-class properties. Hence we propose the Lipschitz margin ratio to incorporate both the inter-class and intra-class properties into metric learning. 

We start with defining the diameter of a metric space:
\begin{defn}
\cite{weaver1999lipschitz} The \emph{diameter} of a metric space $(\mathcal X, \rho)$ is defined as
\begin{equation*}
    \diam (\mathcal X, \rho)  = \sup_{\bm x_i, \bm x_j \in \mathcal X} \rho(\bm x_i, \bm x_j).
\end{equation*}
\end{defn}

The Lipschitz margin ratio is then defined as the ratio between the margin and $\diam(\mathcal X)$ (i.e.~the diameter) or $\diam(\bm  S_{1}) + \diam(\bm  S_{-1})$  (i.e.~the sum of intra-class dispersion), as follows.
\begin{defn}
The \emph{Diameter Lipschitz Margin Ratio} ($\mbox{L-Ratio}^{Diam}$) and the \emph{Intra-Class Dispersion Lipschitz Margin Ratio} ($\mbox{L-Ratio}^{Intra}$) in a metric space $(\mathcal X, \rho)$ are defined as
\begin{equation*}
\begin{split}
    \mbox{L-Ratio}^{Diam} &= \frac{D (\bm S_1, \bm S_{-1})}{\diam(\mathcal X, \rho)}\\ 
    &=\frac{ \min\limits_{\bm x_i \in \bm S_{-1}, \bm x_j\in \bm S_{1} } \rho(\bm x_i, \bm x_j)}{\sup\limits_{\bm x_i, \bm x_j \in \mathcal X} \rho(\bm x_i, \bm x_j)},
\end{split}
\end{equation*}

\begin{equation*}
\begin{split}
    \mbox{L-Ratio}^{Intra} &= \frac{D(\bm S_{1}, \bm S_{-1})}{ \diam(\bm  S_{1}, \rho) + \diam(\bm  S_{-1}, \rho)}\\
    &=\frac{ \min\limits_{\bm x_i \in \bm S_{-1}, \bm x_j\in \bm S_{1} } \rho(\bm x_i, \bm x_j)}{\sup\limits_{\bm x_i, \bm x_j \in \mathcal S_{1}} \rho(\bm x_i, \bm x_j)+\sup\limits_{\bm x_i, \bm x_j \in \mathcal S_{-1}} \rho(\bm x_i, \bm x_j)}.
\end{split}
\end{equation*}
\end{defn}

The relationship between $\mbox{L-Ratio}^{Diam}$ and $\mbox{L-Ratio}^{Intra}$ can be established via the following proposition.

\begin{pro}
\label{pro2}
In a metric space $(\mathcal X, \rho)$, %$1/\mbox{L-Ratio}^{Diam}$ could be upper bounded by $1/\mbox{L-Ratio}^{Intra}+1$:
\begin{equation*}
     \diam(\mathcal X, \rho)  \leq  \diam(\bm  S_{1}, \rho) + \diam(\bm  S_{-1}, \rho)  +D(\bm S_{-1}, \bm S_{1}) 
\end{equation*}
and
\begin{equation*}
\begin{split}
    \frac{1}{\mbox{\quad L-Ratio}^{Diam}} 
    \leq \frac{1}{\mbox{\quad L-Ratio}^{Intra}}+1.
\end{split}
\end{equation*}
\end{pro}
\begin{proof}: See Appendix \ref{proof_a1}
\end{proof}

In this inequality, $\diam(\bm  S_{1}, \rho)$  and $\diam(\bm  S_{-1}, \rho)$ indicate the maximum intra-class distances, and  $D(\bm S_{1}, \bm S_{-1})$ indicates the inter-class margin.
Therefore, this inverse margin ratio penalty will push the learner to select a metric $\rho$ which pulls the instances from the same class closer (small $\sum_{t=1, -1} \diam(\bm  S_{t}, \rho)$) and enlarges the margin between the instances from different classes (large $D(\bm S_{1},\bm S_{-1} )$). 
In a very simple (linearly separable one-dimensional) case, as illustrated in Figure \ref{margin_ratio_line},  $\diam(\mathcal X, \rho)$ can be decomposed into intra-class dispersion ($\diam(\bm  S_{-1}, \rho)$, $\diam(\bm  S_{-1}, \rho)$) and inter-class margin ($D(\bm S_{1}, \bm S_{-1})$) directly.

\begin{figure}
\centering
\includegraphics[width=0.5\textwidth]{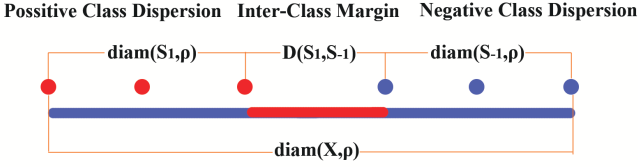}
\caption{An illustration of the relationship between the margin ratio and the intra-/inter-class properties using the indicative linearly separable one-dimensional case as an example. The red solid circles indicate the positive class instances; the blue solid circles indicate the negative class instances.}
\label{margin_ratio_line}
\end{figure}

Then we can bound the Lipschitz margin ratio using the Lipschitz constant and the diameter of metric space:
%%%%%%%%
\begin{pro}
For any $L$-Lipschitz function $f$ satisfying $\forall \bm x_i \in \bm S_{1}, f(\bm x_i) \geq 1$ and $\forall \bm x_j \in \bm S_{-1}, f(\bm x_j) \leq -1$,
\begin{equation*}
    \mbox{L-Ratio}^{Diam} \geq  \frac{2}{L  \diam(\mathcal X, \rho)}, 
\end{equation*}
\begin{equation*}
    \mbox{L-Ratio}^{Intra} \geq  \frac{2}{L  \diam(\bm  S_{1}, \rho) + L \diam(\bm  S_{-1}, \rho) }.
\end{equation*}
\end{pro}
\begin{proof} 
% \begin{equation*}
% \begin{split}
%      \mbox{L-Ratio}^{Diam} &= \frac{\mbox{L-Margin}}{\diam(\mathcal X, \rho)}\\
%      &  \geq  \frac{2}{L  \diam(\mathcal X, \rho)},
% \end{split}
% \end{equation*}
% where the inequality is the result of Proposition \ref{margin_lipcons}.
The inequalities can be obtained by substituting the result of Proposition \ref{margin_lipcons}.
\end{proof}

Based on this proposition, although it is not possible to calculate the exact value of the Lipschitz margin ratio in most cases, we can use $\frac{1}{L  \diam(\mathcal X, \rho)}$ or $\frac{1}{L  \diam(\bm  S_{1}, \rho) + L \diam(\bm  S_{-1}, \rho) }$ as a surrogate.  For example, in the objective function of metric learning by maximizing Lipschitz margin ratio, we can maximize  $\frac{1}{L  \diam(\mathcal X, \rho)}$ or $\frac{1}{L  \diam(\bm  S_{1}, \rho) + L \diam(\bm  S_{-1}, \rho) }$ or equivalently minimize $ L  \diam(\mathcal X, \rho)$ or $L ( \diam(\bm  S_{1}, \rho) + \diam(\bm  S_{-1}, \rho))$.

Furthermore, in some cases we may be more interested in the local properties rather than the global ones (see also Section 4.2). In those cases we can define the {\it local} Lipschitz margin ratio as follows.

\begin{defn}
\label{local_inverse}
The \emph{local Lipschitz margin ratio} with subset $\bm S^l \subseteq \bm S$ and metric $\rho^l \in \mathcal D$  is defined as
\begin{equation*}
    \mbox{Local-Ratio}^{Diam} = \frac{\mbox{L-Margin}}{\diam(\bm S^l, \rho^l) } =  \frac{D(\bm S^l_{1}, \bm S^l_{-1})}{\diam(\bm S^l, \rho^l) },
\end{equation*}
\begin{equation*}
\begin{split}
    \mbox{Local-Ratio}^{Intra}& = \frac{\mbox{L-Margin}}{ \diam(\bm  S_{1}, \rho) + \diam(\bm  S_{-1}, \rho) } \\
    &=  \frac{D(\bm S^l_{1}, \bm S^l_{-1})}{\diam(\bm S^l_{1}, \rho^l)+\diam(\bm S^l_{-1}, \rho^l)},
\end{split}    
\end{equation*}
where $\bm S^l_{k}=\{\bm x_i | t_i = k, \bm x_i \in \bm S^l\}$ indicates the local training set of class $k$ and $k\in\{1,-1\}$.
\end{defn}

\subsection{Learning Bounds of the Lipschitz Margin Ratio}

In the section above, we have defined the Lipschitz margin ratio, which is a measure of model complexity. In this section, we shall establish the effectiveness of the Lipschitz margin ratio through showing the relationship between its lower bound and the generalization ability. 

%%%%The Lipschitz margin ratio appears in the learning bounds of paper \cite{gottlieb2014efficient}.

\begin{defn}
\cite{gottlieb2014efficient} 
For a metric space $(\mathcal X, \rho)$, let $\lambda$ be the smallest number such that every ball in $\mathcal X$
can be covered by $\lambda$ balls of half the radius. Then $\lambda$ is called the \emph{doubling constant} of $\mathcal X$ and the \emph{doubling dimension} of
$\mathcal X$ is $\ddim(\mathcal X) = \log_2 \lambda $.
%A metric is doubling when its doubling dimension is bounded.
\end{defn}
As presented in \cite{gottlieb2014efficient}, a low Euclidean dimension implies a low doubling dimension (Euclidean metrics of dimension $d$ have doubling dimension $O(d)$); a low doubling dimension is more general than a low Euclidean dimension and can be utilized to measure the `dimension' of a general metric space.

\begin{defn}
We say that $\mathcal F$ \emph{$\gamma$-shatters} $\bm x_1, \dots , \bm x_n$, if there exists witness $s_1,\dots, s_n$, such that, for every $\epsilon \in\{\pm 1\}^n$, there exists $f\in \mathcal F$ such that $\forall t \in \{1,\dots, n\}$
\begin{equation*}
    \epsilon_t (f_{\epsilon}(\bm x_t)-s_t) \geq \gamma
\end{equation*}
\emph{Fat-shattering dimension} is defined as follows
\begin{equation*}
\begin{split}
    fat_{\gamma}(\mathcal{F}) =&\max \{n; \exists \bm x_1, \dots, \bm x_n\in \mathcal X, \\
    &s.t.\mbox{ } \mathcal F\mbox{ } \gamma \mbox{ shatters } \bm x_1, \dots, \bm x_n\}.
\end{split}
\end{equation*}
\end{defn}

\begin{thm}
\label{Dtheorem}
\cite{gottlieb2014efficient} 
%Let metric space $(\mathcal X, \rho)$ have doubling dimension $\ddim(\mathcal X)$ and 
Let $\mathcal F$ be the collection of real valued functions over $\mathcal X$ with the Lipschitz constant at most $L$. Define $D= \fat_{1/16}(\mathcal{F})$ and let $P$ be some probability distribution on $\mathcal{X} \times \{-1,1\}$. Suppose that $(x_i, t_i), i=1,\dots, n$ are drawn from $\mathcal{X} \times \{-1,1\}$ independently according to $P$.  Then for any $f \in \mathcal F$ that classifies a sample of size $n$ correctly, we have with probability at least $1-\delta$
\begin{equation*}
\begin{split}
    &P\{(\bm x, t): \sign[f(\bm x)] \neq t\} \\
    &\leq \frac{2}{n} (D\log_2 (34 en/D) \log_2 (578n) + \log_2 (4/\delta)).
\end{split}
\end{equation*}
Furthermore, if $f$ is correct on all but $k$ examples, we have with probability at least $1-\delta$
\begin{equation}
\label{Lipschitz_bound_1}
\begin{split}
    &P\{(\bm x, t): \sign [f(\bm x)] \neq t\} \\ 
    &\leq \frac{k}{n}+ \sqrt{ \frac{2}{n} (D\log_2 (34 en/D) \log_2 (578n) + \log_2 (4/\delta))}.
\end{split}
\end{equation}
%where $c= 8(L \diam(\mathcal X, \rho))^{\ddim(\mathcal X)+1}. $
%\geq 8 (\frac{2}{\mbox{L-Ratio}})^{\ddim(\mathcal X)+1}$.
\end{thm}

\begin{pro}
\label{Dbounds}
In classification problems, when $t_i \in \{-1,1\} $, $L =\sup_{i,j}\frac{2}{\rho(x_i, x_j)}$, then $D= \fat_{1/16}(\mathcal{F})$ can be bounded by the surrogate of Lipschitz Margin Ratio as follows:
\begin{equation}
\begin{split}
   D & \leq \Big{(} 16 L \diam(\mathcal X, \rho) \Big{)}^{\ddim(\mathcal{X})} \\
   & \leq\Big{(} 16 L  (\diam(\bm  S_{1}, \rho) + \diam(\bm  S_{-1}, \rho) ) + 32 \Big{)}^{\ddim(\mathcal{X})}.
\end{split}
\end{equation}
\end{pro}

\begin{proof}
The first inequality has been proved in~\cite{gottlieb2014efficient}. We prove the second inequality here. Because $L =\sup_{i,j}\frac{2}{\rho(x_i, x_j)} = \frac{2}{D(\bm S_{-1}, \bm S_{1})}$, we have
\begin{equation*}
    L D(\bm S_{-1}, \bm S_{1})=2.
\end{equation*}
It follows that
\begin{equation*}
\begin{split}
    L \diam(\mathcal X, \rho) &\leq L ( \diam(\bm  S_{1}, \rho) + \diam(\bm  S_{-1}, \rho) +D(\bm S_{-1}, \bm S_{1}) )\\
    & =  L (( \diam(\bm  S_{1}, \rho) + \diam(\bm  S_{-1}, \rho) ) + 2,
\end{split}
\end{equation*}
where the first inequality is based on Proposition \ref{pro2}.
Meanwhile, because $\ddim(\mathcal{X}) \geq 1$, the second inequality holds. 
\end{proof}

\begin{cor}
Under the condition that $n \geq \frac{D}{34e}$, the following bounds for the surrogate margin ratios holds.
% For any $f \in \mathcal F$ that classifies a sample of size $n$ correctly, we have with probability at least $1-\delta$
% \footnotesize
% \begin{equation*}
% \begin{split}
%     &P\{(\bm x, t): \sign[f(\bm x)] \neq t\} \\
%     &\leq \frac{2}{n} ((16C)^{\ddim(\mathcal{X})}\log_2 (34 en/(16C)^{\ddim(\mathcal{X})}) \log_2 (578n) + \log_2 (4/\delta)).
% \end{split}
% \end{equation*}
% \normalsize
% Furthermore, 
If $f$ is correct on all but $k$ examples, we have with probability at least $1-\delta$
\footnotesize
\begin{equation}
\label{Lipschitz_bound_2}
\begin{split}
    &P\{(\bm x, t): \sign [f(\bm x)] \neq t\} \leq \frac{k}{n}+\\ 
    & \sqrt{ \frac{2}{n} ((16C)^{\ddim(\mathcal{X})} \log_2 (34 en/(16C)^{\ddim(\mathcal{X})}) \log_2 (578n) + \log_2 (4/\delta))},
\end{split}
\end{equation}
\normalsize
where $C =  L \diam(\mathcal X, \rho)$ or $C =  L ( \diam(\bm  S_{1}, \rho) + \diam(\bm  S_{-1}, \rho) ) +2 $. 
\end{cor}

\begin{proof}
Substitute the inequalities of Proposition \ref{Dbounds} into Theorem \ref{Dtheorem}. 
%%%the condition is to guarantee $D_1 \leq D_2$, then $f(D_1) \leq f(D_2)$.  
\end{proof}

% The above learning bound can illustrate the relationship between the surrogate inverse Lipschitz margin ratio $L \diam(\mathcal X, \rho)$ or $L ( \diam(\bm  S_{1}, \rho) + \diam(\bm  S_{-1}, \rho) )$  and the generalization ability, which indicates the gap between empirical error $\frac{k}{n}$ and expected error $P\{(\bm x, t): \sign [f(\bm x)] \neq t\}$.  Reducing the value of surrogate inverse Lipschitz margin ratio  would help reduce the gap between the empirical error and the expected error, which implies the generalization ability of the model would be better. Therefore, based on the learning bound, we can see that minimizing inverse Lipschitz margin ratio would be an effective way to improve the generalization ability and control the model complexity.\\
The above learning bound illustrates the relationship between the generalization error (i.e. the difference between the expected error $P\{(\bm x, t): \sign [f(\bm x)] \neq t\}$ and the empirical error $\frac{k}{n}$) and the surrogate inverse Lipschitz margin ratio $L \diam(\mathcal X, \rho)$ or $L ( \diam(\bm  S_{1}, \rho) + \diam(\bm  S_{-1}, \rho) )$.  Therefore, reducing the value of surrogate inverse Lipschitz margin ratio would help reduce the gap between the empirical error and the expected error, which implies an improvement in the generalization ability of the model. In other words, the learning bound indicates that minimizing inverse Lipschitz margin ratio would be an effective way to enhance the generalization ability and control model complexity.

\section{Metric Learning via Maximizing the Lipschitz Margin Ratio}\label{s:framework}

From previous sections, we have seen that Lipschitz functions have the following desirable properties relevant to metric learning:
\begin{itemize}
  \item (Close relationship with metrics) The definitions of the Lipschitz constant, Lipschitz functions and Lipschitz extensions have natural relationship with metrics.
  \item (Strong representation ability) Lipschitz functions, in particular Lipschitz extensions, could obtain small empirical risks, and hence illustrate the representational capability of Lipschitz functions.
  \item (Good generalization ability) Complexity of Lipschitz functions could be controlled by penalizing the Lipschitz margin ratio.
\end{itemize}

Therefore, it is reasonable for us to conduct metric learning with the Lipschitz functions and control the model complexity by maximizing (the lower bound of) the Lipschitz margin ratio.

\subsection{Learning Framework}

Similarly to other structure risk minimization approaches, we minimize the empirical risk and maximize (the lower bound of) the Lipschitz margin ratio in the proposed framework. To estimate (the lower bound of) the Lipschitz margin ratio, we may either 
\begin{itemize}
  \item use training instances to estimate the Lipschitz constant $L(f)$ and the diameters ${\rm diam}(\mathcal X, \rho)$, and obtain  $\hat L$ and $\hat \diam$; or
  \item adopt the upper bounds of $L$ and ${\rm diam}(\mathcal X, \rho)$ by applying the properties of the classifier $f$ and metric space $(\mathcal X, \rho)$, and obtain $L^s$ and $\diam^s$.
\end{itemize}
The optimization problem could be formulated as follows:
\begin{equation}
\label{frame}
\begin{array}{cc}
\min\limits_{\bm \xi, \bm a, \rho}&  1/\mbox{L-Ratio} + \alpha \sum_{i=1}^N \xi_i \\
s.t. & t_i f(\bm x_i; \bm a, \rho) \geq 1- \xi_i\\
     & \xi_i \geq 0\\
     & i= 1, \dots, N,
\end{array}
\end{equation}
where $N$ indicates the number of training instances; $\bm a$ denotes the parameters of the classification function $f$; $\bm \xi=\{\xi_i\}$ is the hinge loss;  $\alpha>0$ is a trade-off parameter which balances the empirical risk term $\sum_{i=1}^N \xi_i$  and the generalization ability term $1/\mbox{L-Ratio}$. $L(f)$ and $\diam(\mathcal X, \rho)$, $\diam(\bm  S_{1}, \rho)$ and $\diam(\bm  S_{-1}, \rho)$， from the $\mbox{L-Ratio}$ term, will be replaced by either the empirically estimated values $\hat L$ and $\hat \diam$ or the theoretical upper bounds $L^s$ and $\diam^s$.

Empirical estimates of $\hat L$ and $\hat \diam$ can be added as constraints
\begin{equation*}
\frac{f(\bm x_i; \bm a, \rho) - f(\bm x_j; \bm a, \rho) }{\rho (\bm x_i, \bm x_j)} \leq \hat L, 
\end{equation*}
\begin{equation*}
\begin{split}
{\rho (\bm x_i, \bm x_j)} &\leq \hat \diam(\mathcal X, \rho), \quad \mbox{where } x_i \in \bm S, x_j \in \bm S,\\
{\rho (\bm x_i, \bm x_j)} &\leq \hat \diam(\bm  S_{1}, \rho), \quad \mbox{where } x_i \in \bm S_{1}, x_j \in \bm S_{1},\\
{\rho (\bm x_i, \bm x_j)} &\leq \hat \diam(\bm  S_{-1}, \rho),\quad \mbox{where } x_i \in \bm S_{-1}, x_j \in \bm S_{-1}.
\end{split}
\end{equation*}

Then the objective function of minimizing $1/\mbox{L-Ratio}^{Diam}$ becomes 
\begin{equation}
\label{diam}
\begin{array}{cc}
\min\limits_{\bm \xi, \bm a, \rho, \hat L, \hat \diam}&  \hat L   \hat{\diam}(\mathcal X, \rho) + \alpha \sum_{i=1}^N \xi_i ,
\end{array}
\end{equation}
where the penalty of $\hat L   \hat{\diam}(\mathcal X, \rho)$ tries to maximize the inter-class margin (via minimizing $\hat L$) and minimize the overall diameter (via minimizing $\hat{\diam}(\mathcal X, \rho) $). 

The objective function to minimize $1/\mbox{L-Ratio}^{Intra}$ becomes 
\begin{equation*}
%\begin{array}{cc}
\min\limits_{\bm \xi, \bm a, \rho, \hat L, \hat \diam}  \hat L  (\hat \diam(\bm  S_{1}, \rho) +  \hat \diam(\bm  S_{-1}, \rho)) + \alpha \sum_{i=1}^N \xi_i,
%\end{array}
\end{equation*}
or we can minimize an upper bound of $1/\mbox{L-Ratio}^{Intra}$ as
\begin{equation}
\label{intra}
%\begin{array}{cc}
\min\limits_{\bm \xi, \bm a, \rho, \hat L, \hat \diam}  2\hat L  \max(\hat \diam(\bm  S_{1}, \rho) , \hat \diam(\bm  S_{-1}, \rho)) + \alpha \sum_{i=1}^N \xi_i,
%\end{array}
\end{equation}
where the penalty terms of $L  (\hat \diam(\bm  S_{1}, \rho) +  \diam(\bm  S_{-1}, \rho)) $ or $\hat L  \max(\hat \diam(\bm  S_{1}, \rho) , \hat \diam(\bm  S_{-1}, \rho))$ tries to maximize the inter-class margin (via minimizing $\hat L$) and minimize the intra-class dispersion (via minimizing $\hat \diam(\bm  S_{1}, \rho) +  \hat \diam(\bm  S_{-1}, \rho)$ or $\max(\hat \diam(\bm  S_{1}, \rho) , \hat \diam(\bm  S_{-1}, \rho))$) at the same time.

\subsection{Relationship with other Metric Learning Methods}
Some widely adopted metric learning algorithms can be shown as special cases of the proposed framework. 

As presented in Appendix~\ref{appendix:LMML}, based on our framework, the penalty term of LMML \cite{schultz2004learning} could be interpreted as an upper bound of $1/\mbox{L-Ratio}^{Diam}$ margin ratio; and this framework could suggest a reasonable strategy for choosing the target neighbors and the imposter neighbors in LMML. 
%%%%
Also as discussed in Appendix~\ref{appendix:LMNN}, we can see that the penalty term of LMNN \cite{weinberger2009distance} could be interpreted as an upper bound of  $1/\mbox{Local-Ratio}^{Intra}$.

%SVM may be understood as a special case of the above %framework which adopt theoretical upper bound $L^s, %\diam^s$. It is because when using linear classifier %$f(\bm x; \bm a) = \bm x^T \bm a -b$, $\rho$ is always the %Euclidean distance and we have
%\begin{equation*}
%\begin{split}
%    L(f) %&= \frac{|f(\bm x_1; \bm a)-f(\bm x_2; \bm %a)|}{|\bm x_1 - \bm x_2|_2}\\
%    &=\frac{|\bm a^T (\bm x_1 - \bm x_2)|}{|\bm x_1 - \bm %x_2|_2} \\
%    &\leq \frac{|\bm a|_2 |(\bm x_1 - \bm x_2)|_2}{|\bm %x_1 - \bm x_2|_2} (\mbox{H\"older inequality}) \\
%    & = |\bm a|_2,
%\end{split}
%\end{equation*}
%i.e. $L(f)$ is upper bounded by $L^s = |\bm a|_2$. %Meanwhile, $\diam(\mathcal X, \rho)$ is a constant $C>0$ %when always adopting the Euclidean distance. Substituting %$L(f)$ and $\diam(\mathcal X, \rho)$ in the above formula %with $L^s = |\bm a|_2$ and $\diam(\mathcal X, \rho) = C$, %the minimizing of $C|\bm a|_2  + \alpha \sum_{i=1}^n \xi_i %$ just give rise to the optimization objective function of %SVM. 

%\begin{table}[!b]
%\caption{Binary class datasets used in our experiment\footnote{Some meaningless dimension, such as ID, has %been ignored. }}
%\centering
%\begin{tabular}{|c|c|c|c|}
%\hline
%dataset  &features&instances\\ \hline
%Cancer (UCI)& 9& 699 \\ \hline
%%Fourclass (Libsvm)  &2 & 862\\ \hline
%Haberman (UCI) &3 & 306 \\ \hline
%Liver (UCI) &6 &345\\ \hline
%Pima (UCI) &8  &768\\ \hline
%planning & UCI/ Voting & 16 & 435 \\ \hline
%\end{tabular}
%\label{source}
%\end{table}

\subsection{Applying the Framework for Learning the Squared Mahalanobis Metric}\label{ss:special}

We now apply the proposed framework to learn the squared Mahalanobis metric,
\begin{equation*}
    \rho_{\bm M} (\bm x_i, \bm x_j) = (\bm x_i - \bm x_j)^T \bm M (\bm x_i - \bm x_j), \bm M \in \bm M_{+},
\end{equation*}
where $\bm M_{+}$ is the set of positive semi-definite matrices. 
A Lipschitz extension function is selected as the classifier:
\begin{equation}
\begin{split}
  f(\bm x; \bm a, \rho) =& 
  U_{1/2}(\bm x) \\ 
  = &\frac{1}{2}\min_{i=1,\dots,N} (a_i + L \rho_{\bm M}(\bm x,\bm x_i)) + \\
  & \frac{1}{2} \max_{i=1,\dots,N} (a_i - L \rho_{\bm M} (\bm x,\bm x_i)).
\end{split}
\end{equation}
In binary classification tasks, let $t_i\in \{-1,+1\}$ indicate the label of $x_i$, $i=1,\dots, N$.  

Based on the framework of (\ref{frame}) and (\ref{diam}), firstly we propose an optimization formula which penalizes the $\mbox{L-Ratio}^{Diam}$: 
\begin{equation} 
\label{maopt1}
\begin{array}{cc}
\min\limits_{\bm a, \bm \xi,\bm M, \hat \diam,\hat L }&  \hat L  \hat \diam + \alpha \sum_{i=1}^N \xi_i \\
s.t. & \frac{\lvert a_i - a_j \rvert}{\rho_{\bm M} (\bm x_i, \bm x_j)} \leq \hat L \\
     & \rho_{\bm M} (\bm x_i, \bm x_j) \leq \hat \diam\\
     & t_i a_i = 1- \xi_i\\
     & \xi_i \geq 0, \bm M \in \bm M_{+}\\
     & \quad x_i \in \bm S, x_j \in \bm S.
\end{array}
\end{equation}
 
At first glance, the optimization problem seems quite complex. However, based on the smoothness assumption, balanced class assumption ($|\bm S_1|=|\bm S_2|$)  and some equivalent transformations, as illustrated in Appendix~\ref{opt1}, the following optimization problem can be obtained:
\begin{equation}
\label{maopt3}
\begin{array}{cc}
\min\limits_{\bm \xi,  \bm {M'},  d}&  cd + \sum \xi_{ij} \\
s.t. 
& \rho_{\bm {M'}} (\bm x_i, \bm x_j) \geq  2-\xi_{ij} \\
%& \rho_{\bm {M'}} (\bm x_i, \bm x_j) \geq  \xi_{ij}-2 \\
& x_i \mbox{ and } x_j \mbox{ are instance pairs with different labels}\\
&  \rho_{\bm {M'}} (\bm x_m, \bm x_n) \leq d\\
& \xi_{ij} \geq 0,   \bm {M'} \in \bm M_{+}\\
& \quad x_m, x_n \in \bm S.
\end{array}
\end{equation}
Intuitively speaking, the first set of inequality constraints indicate that the distances between samples from different classes should be large; and the third set of inequality constraints indicate that the estimated diameter should be small.

Based on the framework in (\ref{frame}) and (\ref{intra}),  we can also propose an optimization formula which penalizes the upper bound of $\mbox{L-Ratio}^{Intra}$: 
\begin{equation} 
\label{ma2opt1}
\begin{array}{cc}
\min\limits_{\bm a, \bm \xi,\bm M, \hat \diam,\hat L }&  \hat L  \hat \diam + \alpha \sum_{i=1}^N \xi_i \\
s.t. & \frac{\lvert a_i - a_j \rvert}{\rho_{\bm M} (\bm x_i, \bm x_j)} \leq \hat L \\
     & \rho_{\bm M} (\bm x_m, \bm x_n) \leq \hat \diam\\
     & x_m \mbox{ and } x_n \mbox{ are instance pairs with the same label}\\
     & t_i a_i = 1- \xi_i\\
     & \xi_i \geq 0, \bm M \in \bm M_{+}\\
     & \quad x_i, x_j \in \bm S.
\end{array}
\end{equation}
The only difference between (\ref{maopt1}) and (\ref{ma2opt1}) lies on the selected instance pairs to estimate $\hat{\diam}$: (\ref{maopt1}) utilizes all instance pairs to estimate the diameter of all the training instances, while  (\ref{ma2opt1}) utilizes the instances pairs with the same label to estimate the maximum intra-class dispersion. Similarly to the transformations from  (\ref{maopt1}) to  (\ref{maopt3}), the following optimization problem can be obtained:
\begin{equation}
\label{ma2opt3}
\begin{array}{cc}
\min\limits_{\bm \xi,  \bm {M'},  d}&  cd + \sum \xi_{ij} \\
s.t. 
& \rho_{\bm {M'}} (\bm x_i, \bm x_j) \geq  2-\xi_{i} - \xi_{j} \\
& x_i \mbox{ and } x_j \mbox{ are instance pairs with different labels}\\
&  \rho_{\bm {M'}} (\bm x_m, \bm x_n) \leq d\\
& x_m \mbox{ and } x_n \mbox{ are instance pairs with the same label}\\
& \xi_{i} \geq 0,   \bm {M'} \in \bm M_{+}.
\end{array}
\end{equation}

In order to solve (\ref{maopt3}) and (\ref{ma2opt3}) more efficiently, alternating direction methods of multipliers (ADMM) have been adopted (see {\bf Algorithm}~\ref{algo:relgraph}), and the detailed derivation of the ADMM algorithm is presented in Appendix~\ref{ADMM}. 

\begin{algorithm}
\textbf{Input:} \\
${\bm A}_1, {\bm A}_2$\\
\textbf{Initialize:} \\
$\bm M = \bm I, \bm m_1=\bm m_2= \mbox{vector}(\bm M), \bm p= 2- \bm A_1 \bm m_1,$ \\$\bm q = 2 - \bm A_2 \bm m_2, \bm \alpha_{1,2,3,4} = \bm 0$
\begin{algorithmic}
\WHILE{not converged}
\STATE{
1. Update ${\bm p}_{ij}^{t+1}$ using (\ref{update_p}) \\
2. Update ${\bm q}_{ij}^{t+1}$ using bisection search for $t^*$ and Equation \ref{update_q}\\
3. Update ${\bm m}_1^{t+1}$ using (\ref{update_m1})\\
4. Update ${\bm m}_2^{t+1}$ using (\ref{update_m2})\\
5. Update ${\bm m}^{t+1}$ using (\ref{update_m})\\
6. Update the Lagrangian multipliers $\bm \alpha^{t+1}_1$, ${\bm \alpha}^{t+1}_2$, ${\bm \alpha}^{t+1}_3$, ${\bm \alpha}^{t+1}_4$ using (\ref{update_alpha})
}
\ENDWHILE
\end{algorithmic}
\caption{ADMM for (\ref{maopt3})}
\label{algo:relgraph}
\textbf{Output:} ${\bm M}$ \\
\end{algorithm}

\section{Experiments}\label{s:experiments}

To evaluate the performance of our proposed methods, we compare them with four widely adopted distance-based algorithms: Nearest Neighbor (NN), Large Margin Nearest Neighbor (LMNN)~\cite{weinberger2009distance}, Maximally Collapsing Metric Learning (MCML)~\cite{globerson2005metric} and Neighborhood Components Analysis (NCA)~\cite{goldberger2005neighbourhood}. 
Under our framework, we have implemented Lip$^{D}$ (based on the diameter Lipschitz margin ratio), Lip$^{I}$ (based on the intra-class Lipschitz margin ratio), Lip$^{D}$(P) (ADMM-based fast Lip$^{D}$), Lip$^{I}$(P) (ADMM-based fast Lip$^{I}$). 

Our proposed Lip$^{D}$, Lip$^{I}$ are implemented using the cvx toolbox\footnote{http://cvxr.com/} in MATLAB with the solver of SeDuMi \cite{sturm1999using}. The $C$ in our algorithm is fixed at $1$ and the $\lambda$ in the ADMM algorithm is fixed at $1$. The LMNN, MCML and NCA are from the dimension reduction toolbox\footnote{https://lvdmaaten.github.io/drtoolbox/}. 

In the experimente, we focus on the most representative task, binary classification. Eight publicly available datasets from the websites of UCI\footnote{https://archive.ics.uci.edu/ml/datasets.html} and LibSVM\footnote{https://www.csie.ntu.edu.tw/~cjlin/libsvmtools/datasets/binary.html} are adopted to evaluate the performance, namely Statlog/LibSVM Australian Credit Approval (Australian), UCI/LibSVM Original Breast Cancer Wisconsin (Cancer), UCI/LibSVM Pima Indians Diabetes (Diabetes), UCI Echocardiogram (Echo), UCI Fertility (Fertility), LibSVM Fourclass (Fourclass), UCI Haberman's Survival (Haberman) and UCI Congressional Voting Records (Voting). For each dataset, $60\%$ instances are randomly selected as training samples, the rest as test samples. This process is repeated $10$ times and the mean accuracy is reported.  

\begin{table*}
\centering
\caption{Experiment Results (Mean Accuracy), with the best ones in bold and underlined.}
\label{Experiment_results}
\begin{tabular}{c|c|c|c|c|c|c|c|c}
\hline
datasets  &  Lip$^{D}$             &  Lip$^{D}$(P)             & Lip$^{I}$               &  Lip$^{I}$(P)
          &  NN                    &  LMNN                     & MCML                    &  NCA
\\ \hline\hline
Australian&$79.64\pm2.27$                 &$80.04\pm1.92$                    &\underline{$\bm{80.90}\pm1.74$} &$80.29\pm2.15$
          &$79.89\pm1.31$                 &$79.96\pm2.61$                    &$79.89\pm2.30$                  &$79.89\pm1.18$
\\ \hline
Cancer    &$95.30\pm1.12$                 &$94.84\pm0.95$                    &$95.27\pm1.01$                  &$94.84\pm0.95$
          &$95.61\pm0.68$                 &\underline{$\bm{95.41}\pm0.66$}   &$95.37\pm1.14$                  &$94.95\pm1.17$
\\ \hline
Diabetes  &$69.42\pm2.03$                 &$69.38\pm1.59$                    &$69.64\pm2.62$                  &$68.80\pm1.29$
          &$69.46\pm1.22$                 &$69.90\pm1.79$                    &\underline{$\bm{70.03}\pm1.34$} &$68.44\pm2.69$          
\\ \hline
Echo      &$68.00\pm5.49$                 &\underline{$\bm{69.00}\pm6.30$}   &$68.67\pm8.64$                 &$68.67\pm5.92$
          &$65.36\pm2.43$                 &$62.00\pm10.56$                    &$66.33\pm2.92$                  &$66.33\pm4.97$
\\ \hline
Fertility &$79.02\pm4.48$                 &$81.46\pm5.04$                    &$78.05\pm6.60$                  &$80.98\pm3.78$
          &$83.21\pm2.79$                 &\underline{$\bm{84.39}\pm2.36$}   &$83.17\pm5.69$                  &$83.66\pm2.31$
\\ \hline
Fourclass &\underline{$\bm{99.91}\pm0.14$}&\underline{$\bm{99.91}\pm0.14$}   &$99.86\pm0.15$                  &$99.88\pm0.15$
          &$99.87\pm1.14$                 &$99.68\pm0.42$                    &$99.88\pm0.20$                  &$99.68\pm0.62$

\\ \hline
Haberman  &\underline{$\bm{66.42}\pm2.20$}&$66.26\pm3.12$                    &\underline{$\bm{66.42}\pm2.27$} &$65.77\pm2.83$
          &$66.25\pm1.74$                 &$66.26\pm3.12$                    &\underline{$\bm{66.42}\pm2.24$} &$63.66\pm3.93$
\\ \hline
Voting    &$93.37\pm2.29$                 &$92.40\pm1.90$                    &\underline{$\bm{93.83}\pm1.26$} &$92.40\pm1.90$
          &$92.85\pm0.79$                 &$93.31\pm0.72$                    &$92.40\pm1.66$                  &$93.37\pm1.50$
\\ \hline
\# of best& 2 &2 & \textbf{3} & 0 & 0 & 2& 2 & 0\\ \hline
\end{tabular}
\label{table:results}
\end{table*}

As shown in Table~\ref{table:results}, the proposed algorithms Lip achieve the best mean accuracy on four datasets and equally best with MCML on one dataset. The Lip outperforms 1-NN and NCA on seven datasets and LMNN and MCML on five datasets. The only dataset that the Lip performs worse than all other methods is Fertility, in which our method potentially suffers from within-class outliers and hence has a large intra-class dispersion. Apart from this dataset, LMNN or MCML outperforms the Lip by only a small performance gap, less than $0.5\%$. Such encouraging results demonstrate the effectiveness of the proposed framework.

\section{Conclusions and Future Work}\label{s:conclusion}

In this paper, we have presented that the representation ability of Lipschitz functions is very strong and the complexity of the Lipschitz functions in a metric space can be controlled by penalizing the Lipschitz margin ratio. Based on these desirable properties, we have proposed a new metric learning framework via maximizing the Lipschitz margin ratio. An application of this framework for learning the squared Mahalanobis metric has been implemented and the experiment results are encouraging. 

The diameter Lipschitz margin ratio or the intra-class Lipschitz margin ratio in the optimization function is equivalent to an adaptive regularization. In other words, since we encourage samples to stay close within the same class, samples which locate near the class boundary are valued more than those in the center. Therefore, the performance of our method may deteriorate under the existence of outliers and this problem has been reported on the dataset Fertility. We aim to develop more robust methods in our future work.

The local property within a dataset could vary dramatically, and hence it is worthwhile to develop an algorithm based on local Lipschitz margin ratio. One option is to follow the idea of LMNN, learning a general metric but considering different local Lipschitz margin ratio; or we can learn a separate metric on each local area. 

%Since the proposal is a general framework, it will be valuable to check its effectiveness with other metric settings. This will be our future work.

\bibliography{myref_new}
\bibliographystyle{IEEEtran}

\begin{appendix}

\subsection{Proof on Proposition \ref{pro2}}
\label{proof_a1}

\begin{proof}
In any metric space $(\mathcal X,\rho)$, 
let $\bm x_a$ and $\bm x_b$ denote the training instances which satisfy
\begin{equation*}
\begin{split}
  \rho(\bm x_a, \bm x_b) =  \diam(\bm S, \rho) =  \mathop{\argmax}\limits_{\bm x_a, \bm x_b \in \bm S} \rho(\bm x_a, \bm x_b).
\end{split}
\end{equation*}
(1) If $t_a=t_b$,
\begin{equation*}
\begin{split}
  \diam(\bm S, \rho)  & = \rho (\bm x_a, \bm x_b) \\
    & =  \diam(\bm  S_{t_a}, \rho)\\
    & \leq   \diam(\bm  S_{1}, \rho) + \diam(\bm  S_{-1}, \rho)+D(\bm S_{-1}, \bm S_{1}).
\end{split}
\end{equation*}
(2) If $t_a \neq t_b$, let $\bm x_n$ and $\bm x_m$ denote the nearest instances from different classes, i.e.
\begin{equation*}
\rho(\bm x_n, \bm x_m)= D(\bm S_{1}, \bm S_{-1})= \min\limits_{\bm x_i \in \bm S_{-1}, \bm x_j\in \bm S_{+1} } \rho(\bm x_i, \bm x_j),
\end{equation*}
where $ \bm x_n \in \bm S_{t_a}, \bm x_m \in \bm S_{t_b}$.
We can see
\begin{equation*}
\begin{split}
   \diam(\mathcal X, \rho) &  = \rho (\bm x_a, \bm x_b) \\
    & \leq   \rho(\bm x_a, \bm x_n) + \rho (\bm x_n, \bm x_m)   + \rho(\bm x_m, \bm x_b)   \\
    & \leq   \diam(\bm  S_{1}, \rho) + D(\bm S_{-1}, \bm S_{1}) + \diam(\bm  S_{-1}, \rho).
\end{split}
\end{equation*}
Take the definition of $\mbox{L-Ratio}^{Diam}$ and $\mbox{L-Ratio}^{Intra}$:
\begin{equation*}
\begin{split}
  \frac{1}{\mbox{L-Ratio}^{Diam}}  & =
  \frac{\diam(\mathcal X, \rho)}{D (\bm S_1, \bm S_{-1})}  \\
  & \leq  \frac{ \diam(\bm  S_{1}, \rho) + D(\bm S_{-1}, \bm S_{1}) + \diam(\bm  S_{-1}, \rho)}{D(\bm S_{-1}, \bm S_{1})}  \\
  &= \frac{ \diam(\bm  S_{1}, \rho) + \diam(\bm  S_{-1}, \rho)}{D(\bm S_{-1}, \bm S_{1})} + 1 \\
  &= \frac{1}{\mbox{L-Ratio}^{Intra}} + 1 .
\end{split}
\end{equation*}
\end{proof}

\subsection{Properties of Lipschitz Functions}

We can construct Lipschitz functions via the basic ones using the following Lemmas.
\begin{lemma}
\label{Lipschi_add}
(\cite{weaver1999lipschitz}) Let $u, v\in Lip(\mathcal X)$. Then\\
(a) $L(u+ v) \leq L(u) + L(v)$,\\
(b) $L(a u) \leq |a| L(u)$, where $a$ is a constant,\\
(c) $L(\min(u,v)) \leq \max\{L(u),L(v)\}$, where $\min(u,v)$ denotes the pointwise minimum of the functions $u$ and $v$.
\end{lemma}
This lemma illustrates that after the operations of addition, multiplication by constant, minimization and maximization, the results are still Lipschitz functions.
\begin{lemma}
\label{Lipschi_multiply}
(\cite{weaver1999lipschitz}) Let $u, v \in Lip(\mathcal X)$ and $u, v$ are bounded scale-value functions. Then\\
(a) $L(uv) \leq \|u\|_{\infty} L(v) + \|v\|_{\infty} L(u)$, where $\|u\|_{\infty} = \sup_{\bm x} u(\bm x)$. \\
(b) If $\diam (\mathcal X) \leq \infty$, then the product of any two scalar-valued Lipschitz functions is again Lipschitz.
\end{lemma}
This lemma illustrates that after the operations of function multiplication, the results are Lipschitz functions if the basic Lipschitz functions is bounded.
%\begin{lemma}
%Let $\{\mathcal X_k\}_{k=1}^K$ indicates a disjoint partition of set $\mathcal X$ and the distance metric is Euclidean;
%%Let $\forall \bm x_1, \bm x_2 \in \mathds X, d_{\mathcal X}(\bm x_1, \bm x_2)= \|\bm x_1 - \bm x_2\|_2$;
%Let $h_k \in Lip(\mathcal X)$; Let $\mathcal F$
%indicates a set of functions
%$\mathcal F =\{\bm x \rightarrow \sum_k \mathds 1\{\bm x\in \mathcal X_k\} h_k\}$.
%Then if $f$ is continuous, for any $f$, $Lip(f)\leq \max_k Lip(h_k)$.
%\end{lemma}
%This lemma illustrates that we after decomposition .
%These properties would be useful to construct Lipschitz functions via the the basic ones.

\subsection{Relationship between Lipschitz Margin Ratio and LMML \cite{schultz2004learning}}
\label{appendix:LMML}

The Large Margin Metric Learning (LMML) algorithm~\cite{schultz2004learning}  has a close relationship with the proposed framework (\ref{frame}). Based on our proposed framework, the penalty term of LMML could be interpreted as an upper bound of the inverse Lipschitz margin ratio. At the same time, the proposed framework could suggest a reasonable strategy for choosing the target neighbors and the imposter neighbors in LMML.

LMML uses the Mahalanobis metric $D_{\bm M}$, and the classification function of NN is equivalent to the following $f(\bm x)$:
\begin{equation}
\label{class_lmml}
\begin{split}
  f(\bm x) & = D_{\bm M} (\bm x, \bm S_{-1}) -   D_{\bm M} (\bm x, \bm S_{1}) \\
    & =  \min\limits_{a}\{ \rho_{\bm M} (\bm x, \bm x_a) \} -   \min\limits_{b}\{ \rho_{\bm M} (\bm x, \bm x_b)\},
\end{split}
\end{equation}
where $\bm x_a \in \bm S_{-1} $, $\bm x_b \in \bm S_{1}$.

Then LMML adopts an upper bound of $1/\mbox{L-Ratio}^{Diam} \leq L(f) \diam(\mathcal X, D_{\bm M})$ as the penalty term. Because $L(\rho_M(\bm x,\bm x_a)) = 1$, according to Lemma \ref{Lipschi_add}(c),  $L(\min\limits_{a}\{ \rho_{\bm M} (\bm x, \bm x_a) \}) \leq 1$. Then according to Lemma \ref{Lipschi_add}(a), $L(f)$ is bounded by $2$ and
\begin{equation*}
    \begin{split}
       & L(f) \max_{n,m} (\bm x_n - \bm x_m)^T  \bm M (\bm x_n - \bm x_m)\\
       & = L(f) \max_{n,m} \|(\bm x_n - \bm x_m)^T \bm M (\bm x_n - \bm x_m) \|_2 \\
        &  \leq L(f) \max_{n,m} \|\bm x_n - \bm x_m\|_2^2 \|\bm M\|_F\\
        &  \leq C \|\bm M\|_F,
    \end{split}
\end{equation*}
where $C=2 \max_{n,m} \|\bm x_n - \bm x_m\|_2^2$ and $\bm x_n, \bm x_m \in \mathcal X$. The first inequality holds because the matrix Frobenius norm is consistent with the vector $l_2$ norm. Therefore, the Frobenius norm or the squared Frobenius norm may be used as the penalty term.

Based on the above discussion, in this special case, the proposed framework (\ref{frame}) could be represented as
\begin{equation}
\label{LMML_temp}
\begin{array}{cc}
\min\limits_{\bm M, \bm \xi}&   \|\bm M\|_F^2 + \alpha \sum_{i=1}^N \xi_i^o \\
s.t. & t_i f(\bm x_i; \bm a) \geq 1- \xi_i^o\\
     & \xi_i^o \geq 0, \bm M \in \bm M_{+}.
\end{array}
\end{equation}
Then, the constraints of $\rho_{\bm M} (\bm x_i,\bm x_k) - \rho_{\bm M} (\bm x_i, \bm x_j)\geq 1- \xi_i, j\rightarrow i, k \nrightarrow i$ in the optimization problem of LMML serve as a heuristic approximation of $t_i f(\bm x_i; \bm a)\geq 1- \xi_i$. 

In fact, by choosing the target neighbor $\bm x_j$ of $\bm x_i$ as the nearest neighbor within the same class measured via the Euclidean metric and the imposter neighbors $\bm x_k$ as all the instances within the different class, i.e. $j=\argmin_{u} \rho_{\bm M= \bm I} (\bm x_i, \bm x_u)$ and $k\in \{u| \bm x_u \in \bm  S_{-t_i}\}$, $\min\limits_{k}\{ \rho_{\bm M} (\bm x_i, \bm x_k) \} - \rho_{\bm M} (\bm x_i, \bm x_j)$ would be an upper bound of $t_i f(\bm x_i)$. 
This is because
\begin{equation*}
\begin{split}
   t_i f(\bm x_i) & = D_{\bm M} (\bm x_i, \bm S_{-t_i}) -   D_{\bm M} (\bm x_i, \bm S_{t_i}) \\
    & =  \min\limits_{k}\{ \rho_{\bm M} (\bm x_i, \bm x_k) \} - D_{\bm M} (\bm x_i, \bm S_{t_i}) \\
    & \geq    \min\limits_{k}\{ \rho_{\bm M} (\bm x_i, \bm x_k) \} - \rho_{\bm M} (\bm x_i, \bm x_j),
\end{split}
\end{equation*}
where the last inequality holds since $\bm x_j$ is $\bm x_i$'s nearest neighbor within the same class measured via the Euclidean metric and cannot be guaranteed to be the neighbor with in the same class with metric $\bm M$, but $-D_{\bm M} (\bm x_i, \bm S_{t_i}) \geq -\rho_{\bm M} (\bm x_i, \bm x_j) $ always holds. 

Let $t_i f(\bm x_i )'= \min\limits_{k}\{ \rho_{\bm M} (\bm x_i, \bm x_k) \} - \rho_{\bm M} (\bm x_i, \bm x_j)$, then the hinge loss of $f(\bm x')$, $(\max[1-t_i f(\bm x_i )',0 ])$, is the upper bound of the hinge loss of $f(\bm x)$, $(\max[1-t_i f(\bm x_i ),0 ])$, because
\begin{equation*}
    \begin{split}
                   & t_i f(\bm x_i )'=  \min\limits_{k}\{ \rho_{\bm M} (\bm x_i, \bm x_k) \} - \rho_{\bm M} (\bm x_i, \bm x_j) \leq t_i f(\bm x_i )\\
      \Rightarrow  & 1- t_i f(\bm x_i )' \geq 1- t_i f(\bm x_i )\\
      \Rightarrow  & \max[1-t_i f(\bm x_i )',0 ] \geq \max[1-t_i f(\bm x_i ),0 ].
    \end{split}
\end{equation*}
Therefore, the hinge loss  $\xi_i$ obtained by the following optimization problem is the upper bound of $\xi_i^o$ in (\ref{LMML_temp}):
\begin{equation*}
\begin{array}{cc}
\min\limits_{\bm M, \bm \xi}&   \|\bm M\|_F^2 + \alpha \sum_{i=1}^N \xi_i \\
s.t. & t_i f(\bm x_i; \bm a)' \geq 1- \xi_i\\
     & \xi_i \geq 0, \bm M \in \bm M_{+}.
\end{array}
\end{equation*}

The above optimization problem is equivalent to the following one:
\begin{equation*}
\begin{array}{cc}
\min\limits_{\bm M, \bm \xi}&   \|\bm M\|_F^2 + \alpha \sum_{i=1}^N \xi_i \\
s.t. & \rho_{\bm M} (\bm x_i, \bm x_k) - \rho_{\bm M} (\bm x_i, \bm x_j)\geq 1- \xi_i\\
     & \xi_i \geq 0, \bm M \in \bm M_{+},
\end{array}
\end{equation*}
where $\bm x_j$ is $\bm x_i$'s  nearest neighbor within the same class measured via the Euclidean metric and $\bm x_k$ are all the instances within the different class. This is a special case of the optimization problem of LMML. Instead of using a heuristic approximation of the empirical risk, this setting of the target neighbor and the imposter neighbors could guarantee that $\xi_i$ is the upper bound of $\xi_i^o$.

\subsection{Relationship between  Lipschitz Margin Ratio and LMNN \cite{weinberger2009distance}}
\label{appendix:LMNN}

Large Margin Nearest Neighbor (LMNN)~\cite{weinberger2009distance} also has a close relationship with the proposed framework. Similarly to that for LMML, the proposed framework could provide a reasonable strategy for choosing the target neighbors and the imposter neighbors in LMNN. In the following discussion, let $\bm x_j$ be $\bm x_i$'s  nearest neighbor within the same class measured via the Euclidean metric and let $\bm x_k$ be all the instances within the different class of $\bm x_i$. We shall show that the penalty term of LMNN could be interpreted as an upper bound of $1/\mbox{Local-Ratio}^{Intra}$ and $\xi_i$ is also an upper bound of the empirical loss of $\bm x_i$.

LMNN uses the Mahalanobis metric $\rho_{\bm M}$, and the classification function is the same as that of LMML (\ref{class_lmml}).

When the local margin of $\bm x_i$ with metric $\rho_{\bm M}$ is considered, the ideal subset $\bm S^l$ around $\bm x_i$ is $\{\bm x_i, \bm x_m, \bm x_n\}$, where $\bm x_m$ is $x_i$'s nearest neighbor within the same class measured via the metric $\rho_{\bm M}$ and $\bm x_n$ is $x_i$'s nearest neighbor within the different class measured via the metric $\rho_{\bm M}$.  This subset is important for $\bm x_i$ because it determines the classification function of $\bm x_i$. Based on Definition~\ref{local_inverse}, the local inverse Lipschitz margin ratio could be expressed as
\begin{equation*}
    \frac{\diam(\bm S^l, \rho_{\bm M}) }{\mbox{L-Margin}},
\end{equation*}
and based on Proposition \ref{pro2}, it could be bounded as
\begin{equation*}
\begin{split}
  \frac{1}{\mbox{Local-Ratio}^{Intra}} & =  \frac{\diam(\bm S^l_{1}, \rho^l)+\diam(\bm S^l_{-1}, \rho^l)}{\mbox{L-Margin}}\\
    & \leq \frac{1}{2}L(f) \{ \diam(\bm  S^l_{t_i}, \rho_{\bm M}) + \diam(\bm  S^l_{-t_i}, \rho_{\bm M}) \} \\
    & = \frac{1}{2} L(f) \rho_{\bm M}(\bm x_i, \bm x_m) ,
\end{split}
\end{equation*}
where the last equality holds because $\bm S^l= \{\bm x_i, \bm x_m, \bm x_n\}$, so $\bm S^l_{t_i}= \{\bm x_i, \bm x_m\}$, $\bm S^l_{-t_i}= \{\bm x_n\}$ and $\diam(\bm  S^l_{t_i}, \rho_{\bm M})= \rho_{\bm M}(\bm x_i, \bm x_m)$, $\diam(\bm  S^l_{-t_i}, \rho_{\bm M})= 0$. Because $L(f)\leq 2$, we can see
\begin{equation*}
\frac{1}{\mbox{Local-Ratio}^{Intra}}\leq  \rho_{\bm M}(\bm x_i, \bm x_m)  \leq   \rho_{\bm M} (\bm x_i, \bm x_j),
  %\mbox{Local-Ratio}^{Intra}
    % \leq C_1 \rho_{\bm M}(\bm x_i, \bm x_m)  \leq C_1  \rho_{\bm M} (\bm x_i, \bm x_j),
\end{equation*}
%where $C_1$ and $C_2$ are constants, and 
where the second inequality holds because $\bm x_j$ is defined as $\bm x_i$'s  nearest neighbor within the same class measured via the Euclidean metric and $\bm x_m$ may not be the same as $\bm x_j$, thus
\begin{equation*}
    \begin{split}
      \rho_{\bm M}(\bm x_i, \bm x_m)  & = D_{\bm M}(\bm x_i, \bm S_{t_i}) \\
        & = \min\limits_{\bm x_u\in \bm S_{t_i}} \rho_{\bm M}(\bm x_i, \bm x_u)  \\
        & \leq \rho_{\bm M}(\bm x_i, \bm x_j), \forall \bm  x_j \in \bm S_{t_i}.
    \end{split}
\end{equation*}
Therefore, it is reasonable to penalize the sum of the upper bound of the local inverse Lipschitz margin ratios via
\begin{equation*}
    \sum_{i} \rho_{\bm M} (\bm x_i,\bm x_j).
\end{equation*}

Similarly to the discussion of LMML, the strategy of choosing target and imposter neighbors could guarantee that $\xi_i$ is the upper bound of the empirical risk of $\bm x_i$.

The optimization problem based on the proposed framework (\ref{frame}) could be rewritten as
\begin{equation}
\begin{array}{cc}
\min\limits_{\bm M, \bm \xi}& \sum_{i} \rho_{\bm M} (\bm x_i,\bm x_j) + \alpha \sum_{i}  \xi_{i} \\
s.t. &  \rho_{\bm M} (\bm x_i,\bm x_j) -  \rho_{\bm M} (\bm x_i,\bm x_k) \geq 1- \xi_{i} \\
     & \xi_{i} \geq 0, \bm M \in \bm M_{+},
\end{array}
\end{equation}
where $\bm x_j$ is $\bm x_i$'s  nearest neighbor within the same class measured via Euclidean metric and $\bm x_k$ are all the instances within the different class of $\bm x_i$. This is an optimization problem of LMNN with a special strategy for choosing the target neighbor and imposter neighbor.  This strategy could guarantee that $\xi_i$ is the upper bound of the empirical risk.

\subsection{From (\ref{maopt1}) to (\ref{maopt3})}
\label{opt1}
 
To start with, we assume that the intra class area is relatively smooth and $\hat L$ is always determined by instance pairs with different labels, then the optimization problem (\ref{maopt1}) can be written as
\begin{equation} 
\label{maopt2}
\begin{array}{cc}
\min\limits_{\bm a, \bm \xi,\bm M, \hat \diam,\hat L }&  \hat L  \hat \diam + \alpha \sum_{n=1}^N \xi_i \\
s.t. & \frac{\lvert a_i - a_j \rvert}{\rho_{\bm M} (\bm x_i, \bm x_j)} \leq \hat L \\
     & x_i \mbox{ and } x_j \mbox{ are instance pairs }\\ 
     & \mbox{ with different labels}.\\
     & \rho_{\bm M} (\bm x_m, \bm x_n) \leq \hat \diam\\
     & t_m a_m = 1- \xi_m\\
     & \xi_i \geq 0, \bm M \in \bm M_{+}\\
     & \quad x_m, x_n \in \bm S.
\end{array}
\end{equation}
For the squared Mahalanobis metric, we have the following property:
\begin{equation*}
    \forall C,\ C \rho_{\bm M} (\bm x_i, \bm x_j) = \rho_{C \bm M} (\bm x_i, \bm x_j),
\end{equation*}
where $C$ is any constant. 
 
Based on this property, the optimization problem (\ref{maopt2}) is equivalent to the following one:
\begin{equation*}
\begin{array}{cc}
\min\limits_{\bm a, \bm \xi,  \bm M,  \hat L, \hat \diam}&  \hat L\hat \diam + \alpha \sum_{n=1}^N \xi_i \\
s.t. & \lvert a_i - a_j \rvert \leq \rho_{\hat L \bm M} (\bm x_i, \bm x_j)\\
& x_i \mbox{ and } x_j \mbox{ are instance pairs with different labels}\\
     &  \rho_{\hat L \bm M} (\bm x_m, \bm x_n) \leq \hat L\hat \diam\\
     & t_m a_m = 1- \xi_m\\
     & \xi_i \geq 0,   \bm M \in \bm M_{+}\\
     & \quad x_m, x_n \in \bm S.
\end{array}
\end{equation*}
Take $t_m a_m = 1- \xi_m$ into the first constraint, because $x_i$ and $x_j$ are from different classes, we have
\begin{equation*}
\begin{split}
    \lvert a_i - a_j \rvert 
    = \lvert 1-\xi_i - (\xi_j-1) \rvert 
    = \lvert 2- \xi_i - \xi_j \rvert. \\
\end{split}    
\end{equation*}
Therefore, the objective function becomes
\begin{equation*}
\begin{array}{cc}
\min\limits_{\bm \xi,  \bm M,  \hat L, \hat \diam}&  \hat L\hat \diam + \alpha \sum_{n=1}^N \xi_n \\
s.t. 
& \rho_{\hat L \bm M} (\bm x_i, \bm x_j) \geq \lvert 2-\xi_i - \xi_j \lvert\\
& x_i \mbox{ and } x_j \mbox{ are instance pairs with different labels}\\
&  \rho_{\hat L \bm M} (\bm x_m, \bm x_n) \leq \hat L\hat \diam\\
& \xi_i \geq 0,   \bm M \in \bm M_{+}\\
& \quad x_m, x_n \in \bm S,
\end{array}
\end{equation*}
which is equivalent to the following optimization problem:
\begin{equation*}
\begin{array}{cc}
\min\limits_{\bm \xi,  \bm M,  \hat L, \hat \diam}&  \hat L\hat \diam + \alpha \sum_{n=1}^N \xi_n \\
s.t. 
& \rho_{\hat L \bm M} (\bm x_i, \bm x_j) \geq  2-\xi_i - \xi_j \\
& \rho_{\hat L \bm M} (\bm x_i, \bm x_j) \geq  \xi_i + \xi_j - 2 \\
& x_i \mbox{ and } x_j \mbox{ are instance pairs with different labels}\\
&  \rho_{\hat L \bm M} (\bm x_m, \bm x_n) \leq \hat L\hat \diam\\
& \xi_{i} \geq 0,   \bm M \in \bm M_{+}\\
& \quad x_m, x_n \in \bm S.
\end{array}
\end{equation*}

%To simplify the notation, we denote $d=\hat L \hat \diam$, $\bm {M'} = \hat L \bm{M}$, $\xi_{ij}=\xi_i+\xi_j$ and $c=\frac{1}{\alpha}$. 
To simplify the notation, we denote $\xi_{ij}=\xi_i+\xi_j$. 
With the assumption of balanced class, i.e. $|\bm S_1|=|\bm S_2|=\frac{N}{2}$,
we have $\sum_{t_i \neq t_j} \xi_{ij} =  N \sum_{n=1}^N \xi_n $. 
Let $d=\hat L \hat \diam$, $\bm {M'} = \hat L \bm{M}$,  and $c=\frac{1}{\alpha N}$. This turns the optimization problem into:
\begin{equation*}
\begin{array}{cc}
\min\limits_{\bm \xi,  \bm {M'},  d}&  cd + \sum_{i,j=1}^N \xi_{ij} \\
s.t. 
& \rho_{\bm {M'}} (\bm x_i, \bm x_j) \geq  2-\xi_{ij} \\
& \rho_{\bm {M'}} (\bm x_i, \bm x_j) \geq  \xi_{ij}-2 \\
& x_i \mbox{ and } x_j \mbox{ are instance pairs with different labels}\\
&  \rho_{\bm {M'}} (\bm x_m, \bm x_n) \leq d\\
& \xi_{ij} \geq 0,   \bm {M'} \in \bm M_{+}\\
& \quad x_m, x_n \in \bm S.
\end{array}
\end{equation*}
The constraints with respect to $\xi_{ij}$ are $(i) \xi_{ij} \geq 2-\rho_{\bm {M'}} (\bm x_i, \bm x_j) $, $(ii) \xi_{ij} \leq 2+\rho_{\bm {M'}} (\bm x_i, \bm x_j) $ and $(iii) \xi_{ij} \geq 0$. The objective function is to minimize $\xi_{ij}$, based on the objective function,  constraints (iii), constraints (i) and the fact $\rho_{\bm {M'}} (\bm x_i, \bm x_j)\geq 0$, the maximal value of $\xi_{ij}$ would be smaller or equal to $2$. Thus constraints (ii) would always be satisfied. Thus constraints (ii) could be deleted and the optimization problem could be formulated as (\ref{maopt3}).

\subsection{ADMM Algorithm for  (\ref{maopt3}) and  (\ref{ma2opt3})}
\label{ADMM}

The only difference between (\ref{maopt3}) and (\ref{ma2opt3}) lies on the selected instance pairs to estimate $\hat{\diam}$. For simplicity, only the derivation process of ADMM for (\ref{maopt3}) is illustrated here. 

To start with,  (\ref{maopt3}) is as follows
%(\ref{maopt3}) is approximated via the following optimization problem 
\begin{equation*}
\begin{array}{cc}
\min\limits_{\bm \xi,  \bm {M'},  d}&  cd + \sum_{i,j=1}^N \xi_{ij} \\
s.t. 
& \rho_{\bm {M'}} (\bm x_i, \bm x_j) \geq  2-\xi_{ij} \quad \text{for} \  t_i \neq t_j\\
&  \rho_{\bm {M'}} (\bm x_m, \bm x_n) \leq d\\
& \xi_{ij} \geq 0,   \bm {M'} \in \bm M_{+}.
\end{array}
\end{equation*}
Apply the definition of the squared Mahalanobis directly into the constraint:
\begin{equation*}
\begin{array}{cc}
\min\limits_{\bm \xi,  \bm {M'},  d}&  cd + \sum_{i,j=1}^N \xi_{ij} \\
s.t. 
&  (\bm x_i - \bm x_j)  (\bm x_i - \bm x_j)^T \otimes \bm {M'} \geq  2-\xi_{ij} \quad \text{for} \  t_i \neq t_j\\
& (\bm x_m - \bm x_n)  (\bm x_m - \bm x_n)^T \otimes \bm {M'} \geq  d \\
& \xi_{ij} \geq 0,   \bm {M'} \in \bm M_{+},
\end{array}
\end{equation*}
where we define $A \otimes B = \sum_{i,j} A_{ij} \cdot B_{ij}$.

We now stack the columns of $\bm {M'}$ into a vector and call this vector $\bm m$. Similarly, we take the vectorization of $(\bm x_i-\bm x_j)(\bm x_i-\bm x_j)^T$ and $(\bm x_m-\bm x_n)(\bm x_m-\bm x_n)^T$, take their transpose and name them as $\bm A_{1,ij}$ and $\bm A_{2,mn}$, respectively. The optimization problem is then equivalent to
\begin{equation*}
\begin{array}{cc}
\min\limits_{\bm \xi,  \bm {M'},  d}&  cd + \sum_{i,j=1}^N \xi_{ij} \\
s.t. 
&  \xi_{ij} \geq  2- \bm A_{1,ij} \bm m  \quad \text{for} \  t_i \neq t_j\\
& d \geq \bm A_{2,mn} \bm m \\
& \xi_{ij} \geq 0, \bm {M'} \in \bm M_{+},
\end{array}
\end{equation*}
where 
\begin{equation*}
\begin{split}
\bm m&=\text{vector}(\bm {M'}) \in \mathbb{R}^{(p \times p) \times 1},\\
\bm A_{1,ij}&=[\text{vector}((\bm x_i-\bm x_j)(\bm x_i-\bm x_j)^T)]^T, \\
\bm A_{2,mn}&=[\text{vector}((\bm x_m-\bm x_n)(\bm x_m-\bm x_n)^T)]^T,
\end{split}
\end{equation*} 
$p=\dim (\bm {M'})$ and $\bm v=\mbox{vector}(\bm V)$ reshapes any matrix $\bm V \in \mathbb{R}^{a\times b}$ into a vector $\bm v \in \mathbb{R}^{(a\times b)\times 1}$.

Transform this problem into the consensus form \cite{parikh2014proximal}:
\begin{equation*}
\begin{array}{cc}
\min\limits_{\bm \xi,  \bm {M'},  d}&  c \max \limits_{i,j}(q_{ij}) + \sum_{i,j=1}^N \max \limits_{i,j}(0,p_{ij}) +\tilde I_{\bm M_+}(\bm M')\\
s.t. \quad
&  \bm p=2 -\bm A_1 \bm m_1, \quad \bm p \in \mathbb{R}^{(N_1 \times N_2) \times 1}\\
&  \bm q=\bm A_2 \bm m_2, \quad \bm q \in \mathbb{R}^{(N \times N) \times 1}\\
&  \bm m_1 = \bm m_2 = \bm m, \quad \bm m_1, \bm m_2 , \bm m \in \mathbb{R}^{(p \times p) \times 1},
\end{array}
\end{equation*}
where $\bm A_1 \in \mathbb{R}^{(N_1 \times N_2) \times (p \times p)}$ consists of $(N_1 \times N_2)$ blocks of $\bm A_{1,ij}$ and $\bm A_2 \in \mathbb{R}^{(N \times N) \times (p \times p)}$ consists of $(N \times N)$ blocks of $\bm A_{2,mn}$. Here $N_1$ and $N_2$ are the number of instances in class 1 and 2 respectively. 
$\tilde I _C (x) = 
\begin{cases}
    0, \ x \in C \\
    \infty, \ x \not\in C
\end{cases}$.

The Augmented Lagrangian function of the above optimization problem becomes
\begin{equation*}
\begin{split}
& L_\mu(\bm \alpha_1,\bm \alpha_2,\bm \alpha_3,\bm \alpha_4,\bm p, \bm q, \bm m_1, \bm m_2, \bm {M'})\\
=&c \max \limits_{i,j}(q_{ij}) + \sum_{i,j=1}^N \max \limits_{i,j}(0,p_{ij}) +\tilde I_{\bm M_+}(\bm M')+\\
& \bm \alpha_1^T(\bm m_1 - \bm m) +\bm \alpha_2^T(\bm m_2 - \bm m) + \\
& \bm \alpha_3^T(\bm p + \bm A_1 \bm m_1 -2) +\bm \alpha_4^T(\bm q - \bm A_2 \bm m_2)+ \\
&\frac{\mu}{2} {\lvert \lvert \bm m_1 - \bm m \rvert \rvert}_2^2 +\frac{\mu}{2} {\lvert \lvert \bm m_2 - \bm m \rvert \rvert}_2^2 + \\
&\frac{\mu}{2} {\lvert \lvert \bm p + \bm A_1 \bm m_1 -2 \rvert \rvert}_2^2 +\frac{\mu}{2} {\lvert \lvert \bm q - \bm A_2 \bm m_2 \rvert \rvert}_2^2, 
\end{split}    
\end{equation*}
where $\bm \alpha_1 \in \mathbb{R}^{(p \times p) \times 1}$, $\bm \alpha_2 \in \mathbb{R}^{(p \times p) \times 1}$, $\bm \alpha_3 \in \mathbb{R}^{(N_1 \times N_2) \times 1}$, $\bm \alpha_4 \in \mathbb{R}^{(N \times N) \times 1}$ are the Lagrangian multipliers and $\mu \in \mathbb{R}^{1}$ is the penalty parameter.

We apply the Alternating Direction Method of Multipliers algorithm (ADMM) to solve this problem. Specifically, we minimize $\bm p, \bm q, \bm m_1, \bm m_2, \bm {M'}$ respectively by fixing other variables and then update $\bm \alpha_1, \bm \alpha_2 ,\bm \alpha_3 ,\bm \alpha_4.$\\
(1) Update $p_{ij}$
\begin{equation*}
\begin{split}
\min \limits_{p_{ij}}L_\mu 
&\Leftrightarrow \min \limits_{p_{ij}}
\max(0,p_{ij}) + \bm \alpha_3^T p_{ij} +  \frac{\mu}{2} {\lvert \lvert p_{ij} + \bm A_{1,ij} \bm m_1 -2 \rvert \rvert}_2^2
%\frac{\mu}{2} {\lvert \lvert \bm p + \bm A_1 \bm m_1 -2 \rvert \rvert}_2^2.
\end{split}
\end{equation*}

According to the proposition in \cite{ye2011efficient},
\begin{equation*}
    S_\lambda(\omega)=\arg \min \limits_x \lambda \max(0,x)+ \frac{1}{2} {\lvert \lvert x- \omega \rvert \rvert} _2^2
\end{equation*}
has the solution
\begin{equation*}
S_\lambda(\omega)=
    \begin{cases}
      \omega - \lambda & \text{if} \ \omega > \lambda \\
      0 & \text{if} \ 0  \leq \omega \leq \lambda \\
      \omega & \text{if} \ \omega < 0.
    \end{cases}
\end{equation*}
Our minimization function can thus be formulated as
\begin{equation*}
\begin{split}
\min \limits_{p_{ij}}L_\mu 
&\Leftrightarrow \min \limits_{p_{ij}}  \max(0,p_{ij}) + \frac{\mu}{2} 
%{\lvert \lvert \bm p - (2- \bm A_1 \bm m_1 -\frac{\bm \alpha_3}{\mu}) \rvert \rvert}_2^2
{\lvert \lvert p_{ij} - (2- \bm A_{1,ij} \bm m_1 -\frac{\bm \alpha_{3,ij}}{\mu}) \rvert \rvert}_2^2\\
&\Leftrightarrow S_{\frac{1}{\mu}}(2- \bm A_{1,ij} \bm m_1 -\frac{\bm \alpha_{3,ij}}{\mu})
%S_{\frac{1}{\mu}}(2- \bm A_1 \bm m_1 -\frac{\bm \alpha_3}{\mu}).
\end{split}
\end{equation*}
Hence we have
\footnotesize
\begin{equation}
\label{update_p}
p_{ij}^{t+1}=
    \begin{cases}
    2- \bm A_{1,ij} {\bm m}_{1}^t -\frac{{\bm \alpha}^t_{3,ij} + 1}{\mu} & \text{if} \  2- \bm A_{1,ij} {\bm m}^t_1 -\frac{{\bm \alpha}^t_{3,ij}}{\mu} > \frac{1}{\mu} \\
      %2- \bm A_{1,ij} {\bm m}_{1,ij}^t -\frac{{\bm \alpha}^t_{3,ij} + 1}{\mu} & \text{if} \  2- \bm A_1 {\bm m}^t_1 -\frac{{\bm \alpha}^t_3}{\mu} > \frac{1}{\mu} \\
      0 & \text{if} \ 0 \leq 2- \bm A_1 {\bm m}^t_1 -\frac{{\bm \alpha}^t_3}{\mu}  \leq \frac{1}{\mu} \\
      2- \bm A_{1,ij} {\bm m}^t_1 -\frac{{\bm \alpha}^t_{3,ij}}{\mu} &  \text{if} \ 2- \bm A_{1,ij} {\bm m}^t_1 -\frac{{\bm \alpha}^t_{3,ij}}{\mu} < 0
    \end{cases}
\end{equation}
\normalsize
(2) Update $\mathit{q_{ij}}$
\begin{equation*}
\begin{split}
\min \limits_{q_{ij}}L_\mu 
&\Leftrightarrow \min \limits_{q_{ij}}  c\max_{i,j}(q_{ij}) + \bm \alpha_4^T q_{ij} + \frac{\mu}{2} {\lvert \lvert q_{ij} - \bm A_{2,ij}\bm m_2 \rvert \rvert}_2^2.
\end{split}
\end{equation*}

According to \cite{parikh2014proximal}, the optimization function
$$\min_x \max_i x_i + \frac{1}{2 \lambda} {\lvert \lvert x-v \rvert \rvert}_2^2 $$
can be written as
\begin{equation*}
\begin{array}{cc}
    \min_x t + \frac{1}{2 \lambda} {\lvert \lvert x-v \rvert \rvert}_2^2\\
    s.t. \quad
    x_i \leq t \quad i = 1, \cdots, n.
\end{array}
\end{equation*}
The optimal value $t^\star$ needs to satisfy the condition
$$    \sum_{i=1}^n \frac{1}{\lambda} \max(0,v_i - t^\star)=1, $$
and this equation can be solved by bisection. Then, the optimal $x^\star$ can be obtained as
$$x_i^\star = \min(t^\star,v_i).$$
Therefore, we rewrite our objective function as follows:
\begin{equation*}
\begin{split}
\min \limits_{q_{ij}}L_\mu 
&\Leftrightarrow \min \limits_{q_{ij}} \max(q_{ij}) + \frac{\mu}{2c}
{\lvert \lvert q_{ij} - (\bm A_{2,ij} \bm m_2 - \frac{\bm \alpha_{4,ij}}{\mu}) \rvert \rvert}_2^2.
%{\lvert \lvert \bm q - (\bm A_2 \bm m_2 - \frac{\bm \alpha_4}{\mu}) \rvert \rvert}_2^2.
\end{split}
\end{equation*}
Hence
\begin{equation}
\label{update_q}
%q_{ij}^{t+1}=\min(t^\star,\bm A_{2,ij} {\bm m}_{2,ij}^t - \frac{{\bm \alpha}_{4,ij}^t}{\mu}),
q_{ij}^{t+1}=\min(t^\star,\bm A_{2,ij} {\bm m}_{2}^t - \frac{{\bm \alpha}_{4,ij}^t}{\mu}),
\end{equation}
and $t^\star$ satisfies
%$$ \sum_{i,j=1}^N \frac{\mu}{c}(\bm A_{2,ij} {\bm m}_{2,ij}^t - \frac{{\bm \alpha}_{4,ij}^t}{\mu} - t^\star) =1. $$
$$ \sum_{i,j=1}^N \frac{\mu}{c}(\bm A_{2,ij} {\bm m}_{2}^t - \frac{{\bm \alpha}_{4,ij}^t}{\mu} - t^\star) =1. $$
(3) Update $\bm{m}_1$\\
\begin{equation*}
\begin{split}
\min \limits_{\bm m1}L_\mu 
\Leftrightarrow \min \limits_{\bm m1} & \quad {\bm \alpha}_1^T {\bm m}_1 +{\bm \alpha}_3^T {\bm A}_1 {\bm m}_1 + \\
& \frac{\mu}{2} {\lvert \lvert {\bm m}_1 - {\bm m} \rvert \rvert}_2^2 + \frac{\mu}{2} {\lvert \lvert \bm p + \bm A_1 \bm m_1 -2 \rvert \rvert}_2^2.
\end{split}
\end{equation*}

Take the derivative with respect to $\bm m_1$, we get
$$ \mu({\bm A}_1^T {\bm A_1} + {\bf I}) {\bm m}_1^\star + {\bm \alpha}_1 + {\bm A}_1^T {\bm \alpha}_3 - \mu {\bm m} + \mu {\bm A}_1^T {\bm p} - 2 \mu {\bm A}_1^T \bf{1} =0, $$
where $\bf I$ is the identity matrix and $\bf 1$ is the vector with all components being 1. Hence, we update $\bm m_1 $ as follows:
\footnotesize
\begin{equation}
\label{update_m1}
{\bm m}_1^{t+1} = ({\bm A}_1^T {\bm A}_1 + {\bf I})^{-1} ({\bm m}^t - \frac{{\bm \alpha}_1^t  + {\bm A}_1^T {\bm \alpha}_3^t + \mu {\bm A}_1^T {\bm p}^{t+1} - 2\mu {\bm A}_1^T {\bf 1}}{\mu}).
\end{equation}
\normalsize
We can save $({\bm A}_1^T {\bm A}_1 + {\bf I})^{-1}$ in the memory so as to improve the computational efficiency. \\
(4) Update $\bm{m}_2$\\
\begin{equation*}
\begin{split}
\min \limits_{\bm m2}L_\mu 
\Leftrightarrow \min \limits_{\bm m2}& \quad {\bm \alpha}_2^T \bm m_2 - {\bm \alpha}_4^T {\bm A}_2 {\bm m}_2 +\\ &\frac{\mu}{2} {\lvert \lvert \bm m_2 - \bm m \rvert \rvert}_2^2 + \frac{\mu}{2} {\lvert \lvert \bm q - \bm A_2 \bm m_2 \rvert \rvert}_2^2.
\end{split}
\end{equation*}

Take the derivative with respect to $\bm m_2$, we get
$$ \mu({\bm A}_2^T {\bm A_2} + {\bf I}) {\bm m}_2^\star + {\bm \alpha}_2 - {\bm A}_2^T {\bm \alpha}_4 - \mu {\bm m} - \mu {\bm A}_2^T {\bm q} =0 .$$
Update $\bm m_2 $ as follows:
\begin{equation}
\label{update_m2}
{\bm m}_2^{t+1} = ({\bm A}_2^T {\bm A}_2 + {\bf I})^{-1} ({\bm m}^t + \frac{ {\bm A}_2^T {\bm \alpha}_4^t + \mu {\bm A}_2^T {\bm q}^{t+1} - \bm{\alpha}_2^t}{\mu}).
\end{equation}
(5) Update $\bm{M'}$ \textit{ (and hence} $\bm m)$\\
\begin{equation*}
\begin{array}{cc}
    \min \limits_{\bm{M'}/\bm m} \tilde I_{\bm M_+}(\bm M') + \bm{\alpha}_1^T({\bm m}_1 - \bm m) + \bm{\alpha}_2^T({\bm m}_2 - \bm m) \\
    + \frac{\mu}{2} {\lvert \lvert {\bm m}_1 - \bm m \rvert \rvert}_2^2 + \frac{\mu}{2} {\lvert \lvert {\bm m}_2 - \bm m \rvert \rvert}_2^2.
\end{array}
\end{equation*}  

% Take the derivative with respect to $\bm m $, we get
% \begin{equation*}
% \begin{array}{cc}
% -\bm{\alpha}_1 -\bm{\alpha}_2 + \mu ({\bm m}^\star_1 - \bm m) + \mu ({\bm m}^\star_2 - \bm m)=0\\
% \bm {M}^{'\star} \in \bm{M_+}
% \end{array}
% \end{equation*}

Hence, update $\bm m$ as
\begin{equation}
\label{update_m}
\begin{split}
{\bm m}^{t+1} = &\prod \nolimits_{{\bm M}_+} \big(\mbox{matrix} (\frac{{\bm m}_1^{t+1} + {\bm m}_2^{t+1} }{2} + \frac{{\bm \alpha}^t_1 +{\bm \alpha}^t_2 }{2 \mu}) + \\
&\mbox{matrix} (\frac{{\bm m}_1^{t+1} + {\bm m}_2^{t+1} }{2} + \frac{{\bm \alpha}^t_1 +{\bm \alpha}^t_2 }{2 \mu})'\big)/2,
\end{split}
\end{equation}
where $\bm V=\mbox{matrix}(\bm v)$ is the reverse operation of $\bm v = \mbox{vector}(\bm V) $ 
and it reshapes a vector $v\in \mathbb{R}^{(p \times p) \times 1}$ into a matrix $\bm V \in \mathbb{R}^{p \times p}$. 
$\prod \nolimits_{{\bm M}_+} $ denotes the projection of a symmetric matrix onto the positive semidefinite cone ${\bm M}_+$.\\
(6) Update $\bm \alpha$
\begin{equation}
\label{update_alpha}
\begin{split}
\bm \alpha^{t+1}_1 & = \bm \alpha^{t}_1 + \mu (\bm m_1^{t+1} - \bm m^{t+1})\\
\bm \alpha^{t+1}_2 & = \bm \alpha^{t}_2 +\mu (\bm m_2^{t+1} - \bm m^{t+1} )\\
\bm \alpha^{t+1}_3 &= \bm
\alpha^{t}_3 +\mu (\bm p^{t+1} +\bm A_1 \bm m_1^{t+1} -2)\\
\bm \alpha^{t+1}_4 & = \bm \alpha^{t}_4 +\mu (\bm q^{t+1} - \bm A_2 \bm m_2^{t+1}).
\end{split}
\end{equation}

%\subsection{Boxplots for Each Datasets}
%
%\begin{figure}[htbp]
%\centering
%\includegraphics[width=0.243\textwidth]{1.eps}%
%\includegraphics[width=0.243\textwidth]{4.eps}\\
%\includegraphics[width=0.243\textwidth]{8.eps}%
%\includegraphics[width=0.243\textwidth]{6.eps}%
%\caption{Boxplots for `Australian', `Echo', `Voting' and `Fourclass'.}
%\end{figure}
%
%\begin{figure}[htbp]
%\centering
%\includegraphics[width=0.243\textwidth]{7.eps}%
%\includegraphics[width=0.243\textwidth]{2.eps}\\
%\includegraphics[width=0.243\textwidth]{3.eps}%
%\includegraphics[width=0.243\textwidth]{5.eps}%
%\caption{Boxplots for `Haberman', `Cancer', `Diabetes' and `Fertility'}
%\end{figure}

\end{appendix}

\end{document}